\documentclass[11pt]{article}
\usepackage[utf8]{inputenc} 
\usepackage{booktabs}       
\usepackage{nicefrac}       
\usepackage{natbib}

\usepackage{caption}

\usepackage{algorithm}
\usepackage{algorithmic}
\usepackage{subcaption}
\usepackage{bbm}
\usepackage{bm}
\usepackage{amsfonts,amsmath,amssymb}
\usepackage{graphicx}
\usepackage{xspace}
\usepackage{boxedminipage}
\usepackage{framed}
\usepackage{url}
\usepackage{float}
\usepackage{paralist,enumitem}
\usepackage{accents}
\usepackage{setspace}
\usepackage{mathrsfs}
\usepackage[cm]{fullpage}%
\usepackage{scalerel}%
\usepackage[figurewithin=section]{caption}
\usepackage{subcaption}
\usepackage{colortbl}
\usepackage{hhline}
\usepackage{tikz}
\usepackage{multirow}
\usepackage{tablefootnote}
\usepackage{pifont}
\usetikzlibrary{decorations.pathreplacing,angles,quotes}
\usepackage[normalem]{ulem}
\usepackage{afterpage}
\usepackage{wrapfig}
\usepackage[margin=1in]{geometry}
\usepackage{mleftright}%

\usepackage{hyperref}
\hypersetup{%
   breaklinks,%
   colorlinks=true,%
   linkcolor=blue,%
   urlcolor=black,%
   citecolor=blue
}
\usepackage{xcolor}

\numberwithin{figure}{section}
\numberwithin{table}{section}
\numberwithin{equation}{section}%

\usepackage{titlesec}%
\titlelabel{\thetitle. }%

\titlespacing*{\section}
{0pt}{8pt plus 2 pt minus 2 pt}{4pt plus 2pt minus 2pt}
\titlespacing*{\subsection}
{0pt}{6pt plus 2 pt minus 2 pt}{4pt plus 2pt minus 2pt}
\titlespacing*{\subsubsection}
{0pt}{4pt plus 2 pt minus 2 pt}{4pt plus 2pt minus 2pt}

\newcommand{\RNum}[1]{\uppercase\expandafter{\romannumeral #1\relax}}

\usepackage{tcolorbox}

\usepackage[amsmath,thmmarks]{ntheorem}%
\theoremstyle{plain}%
\newtheorem{theorem}{Theorem}[section]

\newtheorem{lemma}[theorem]{Lemma}

\newtheorem{corollary}[theorem]{Corollary}
\newtheorem{claim}[theorem]{Claim} 
\newtheorem{observation}[theorem]{Observation}

\newtheorem{definition}[theorem]{Definition}

\theoremstyle{plain}%
\theoremheaderfont{\bf} \theorembodyfont{\upshape}%
\newtheorem{remark}[theorem]{Remark}%
\newtheorem*{remark:unnumbered}[theorem]{Remark}%
\theoremheaderfont{\em}%
\theorembodyfont{\upshape}%
\theoremstyle{nonumberplain}%
\theoremseparator{}%
\theoremsymbol{\myqedsymbol}%
\newtheorem{proof}{Proof:}%

\newcommand{\myqedsymbol}{$\square$}

\def\compactify{\itemsep=0pt \topsep=0pt \partopsep=0pt \parsep=0pt}
\let\latexusecounter=\usecounter

{\begin{itemize}%
\setlength{\itemsep}{0pt}%
\setlength{\topsep}{0pt}%
\setlength{\partopsep}{0pt}%
\setlength{\parskip}{0pt}}%
{\end{itemize}}

\newcommand{\eps}{\epsilon}%
\def\bar{\overline}

\def\script#1{\mathcal{#1}}

\def\E{\mathbf{E}}

\newcommand{\poly}{\mathop{\mathrm{poly}}}%

\def\C{\mathcal{C}}

\def\sA{\script{A}}

\def\O{\script{O}}

\newcommand{\xmark}{\ding{55}}

\newcommand{\dist}{\mathrm{dist}}

\newcommand{\Tr}{\mathsf{Tr}}
\newcommand{\SKALG}{\mathsf{ALG}}
\newcommand{\CS}{\mathsf{CS}}

\newcommand{\row}{\mathsf{row}}
\newcommand{\col}{\mathsf{col}}
\newcommand\norm[1]{\left\lVert#1\right\rVert}
\newcommand{\abs}[1]{\left|#1\right|}

\definecolor{mygray}{gray}{0.5}

\DeclareMathOperator{\rankk}{rank-\mathnormal{k}}
\def\rowsp{\mathrm{rowsp}}

\def\nnz{\mathrm{nnz}}
\def\train{\mathrm{train}}

\def\rank{\mathrm{rank}}

\DeclareMathOperator{\Span}{\mathsf{span}}
\DeclareMathOperator{\proj}{proj}

\newcommand{\sS}{\script{S}}

\DeclareMathOperator*{\argmax}{arg\,max}
\DeclareMathOperator*{\argmin}{arg\,min}

\def\LRA{\mathrm{LRA}}

\title{Learning the Positions in CountSketch}

\allowdisplaybreaks 

\author{%
    Simin Liu\thanks{Carnegie Mellon University; \href{siminliu@andrew.cmu.edu}{siminliu@andrew.cmu.edu}}
\and 
    Tianrui Liu\thanks{Nankai University; \href{1711435@mail.nankai.edu.cn}{1711435@mail.nankai.edu.cn}}
\and
  	Ali Vakilian\thanks{Toyota Technological Institute at Chicago; \href{vakilian@ttic.edu}{vakilian@ttic.edu}.} 
\and
    Yulin Wan\thanks{Anhui University; \href{Y91714026@stu.ahu.edu.cn}{Y91714026@stu.ahu.edu.cn}}
\and
    David P. Woodruff\thanks{Carnegie Mellon University; \href{dwoodruf@andrew.cmu.edu}{dwoodruf@andrew.cmu.edu}}
}

\begin{document}
\date{}
\maketitle
\begin{abstract}
We consider sketching algorithms which first quickly compress data by multiplication with a random sketch matrix, and then apply the sketch to quickly solve an optimization problem, e.g., low rank approximation. In the learning-based sketching paradigm proposed by~\cite{indyk2019learning}, the sketch matrix is found by choosing a random sparse matrix, e.g., the CountSketch, and then updating the values of the non-zero entries by running gradient descent on a training data set. Despite the growing body of work on this paradigm, a noticeable omission is that the {\it locations} of the non-zero entries of previous algorithms were fixed, and only their values were learned. In this work we propose the first learning algorithm that also optimizes the locations of the non-zero entries. We show this algorithm gives better accuracy for low rank approximation than previous work, and apply it to other problems such as $k$-means clustering for the first time. We show that our algorithm is provably better in the spiked covariance model and for Zipfian matrices. We also show the importance of the sketch monotonicity property for combining learned sketches. Our empirical results show the importance of optimizing not only the values of the non-zero entries but also their positions. \end{abstract}

\section{Introduction}
A recent work of~\cite{indyk2019learning} investigated learning-based sketching algorithms for the low-rank approximation problem. A sketching algorithm is a method of constructing approximate solutions for optimization problems via summarizing the data. In particular, \textit{linear} sketching algorithms compress data by multiplication with a sparse ``sketch matrix'' and then use just the compressed data to find an approximate solution. Generally, this technique results in much faster or more space efficient algorithms for a fixed approximation error. 
The work of \cite{indyk2019learning} shows it is possible to learn sketch matrices for low rank approximation (LRA) with better average performance than classical sketches.

In this model, we assume inputs come from an {\em unknown} distribution and learn a sketch matrix with strong expected performance over the distribution. This distributional assumption is often realistic---there are many situations where a sketching algorithm is applied to a large batch of related data. For example, genomics researchers might sketch DNA from different individuals, which is known to exhibit strong commonalities.
The high-performance computing industry has also begun exploring sketching. Recently, researchers at NVIDIA showed interest in creating standard implementations of sketching algorithms for CUDA, a widely used GPU library. They investigated the (classical) sketched singular value decomposition (SVD), but found that the solutions were not exact enough across a spectrum of inputs. This is precisely the issue addressed by the learned sketch paradigm where we optimize for ``good'' average performance across a range of inputs.

While promising results have been shown using previous learned sketching techniques, notable gaps remain. In particular, all previous methods work by initializing the sketching matrix with a random sparse matrix, e.g., each column of the sketching matrix has a single non-zero value chosen at a uniformly random position. Then, the {\it values} of the non-zero entries are updated by running gradient descent on a training data set. However, the {\it locations} of the non-zero entries are held fixed throughout the entire training process. 

Clearly this is sub-optimal. Indeed, suppose the input matrix $A$ is an $n \times d$ matrix with first $d$ rows equal to the $d \times d$ identity matrix, and remaining rows equal to $0$. A random sketching matrix $S$ with a single non-zero per column is known to require $m = \Omega(d^2)$ rows in order for $S \cdot A$ to preserve the rank of $A$~\citep{nn14}; this follows by a simple birthday paradox argument. On the other hand, it is clear that if $S$ is a $d \times n$ matrix with first $d$ rows equal to the identity matrix, then $\|S \cdot Ax\|_2 = \|Ax\|_2$ for all vectors $x$, and so preserves not only the rank of $A$ but all important spectral properties. A random matrix would never choose the non-zero entries in the first $d$ columns of $S$ so perfectly, whereas an algorithm trained to optimize the locations of non-zero entries would notice and correct for this. This is precisely the gap in our understanding that we seek to fill.

\paragraph{Our Results.}
In this work, in addition to learning the values of the non-zero entries, we learn the locations of the non-zero entries (Section~\ref{sec:learn-countsketch}). We also bound the generalization error for low rank approximation for inputs coming from a spiked covariance or Zipfian distribution using our proposed learned sketches (Section~\ref{sec:local-opt}). We also apply our technique to $k$-means clustering, where learned sketches had previously not been considered. 

We show strong empirical results for five datasets across three domains (image, natural language, graph). Our method outperforms classical sketches for both problems and significantly for low rank approximation, $70\%$ (Section~\ref{sec:evaluation}). 

We note that for each optimization problem we consider, we can guarantee that we do no worse than a random sketch, simply by running the algorithm with a random sketch in parallel, evaluating the cost of the output of our algorithm and that of the output of the random sketch based algorithm, and choosing the better of the two. We note that, to obtain time-optimal bounds, additional sketching can be used to evaluate the cost of the output. 

Following \cite{indyk2019learning}, we also show for a number of problems one can simply adjoin a random sketching matrix to the learned sketching matrix rather than running two separate algorithms (the learned sketch and random sketch), which asymptotically has no worst-case advantages. However, by the {\em sketch monotonicity property}---the quality of our final output never decreases when we add rows to the sketch matrix---this is especially useful for combining a learned sketch and a random sketch (Section~\ref{sec:sketch-monotonicity}).   

\paragraph{Additional Related Work.}
In the last few years, there has been much work on leveraging machine learning technique to improve classical algorithms; we only mention a few examples here which are of the sketch-based kind that we consider. One related body of work is data-dependent dimensionality reduction, such as an approach for pair-wise/multi-wise similarity preservation for indexing big data \citep{learning_hash}, learned sketching for streaming problems
\citep{jiang2020learningaugmented,cohen2020composable,edenlearning}, learned algorithms for nearest neighbor search \citep{dong2019learning}, 
and a method for learning linear projections for general applications~\citep{numax}.  
While we also learn linear embeddings, our embeddings are optimized for specific applications (low rank approximation, $k$-means). In fact, one of our central challenges is that the theory and practice of learned sketches generally needs to be tailored to each application.

Our work builds off of the work of~\citeauthor{indyk2019learning}, which introduced gradient descent optimization for low rank approximation, but a major difference is that we also optimize the locations of the non-zero entries in the sketching matrix. 
\section{Preliminaries}
We now introduce our optimization problems and review the learned sketch paradigm. 
\paragraph{Notation.} Denote the SVD of $A$ by $A = U \Sigma V^{\top}$. Define $\left[A\right]_k = U_k \Sigma_k V^{\top}_k$ as the optimal $\rankk$ approximation to $A$, computed by the truncated SVD. Also, define the Moore-Penrose pseudo-inverse of $A$ to be $A^{\dagger} = V \Sigma^{-1} U^{\top}$, where $\Sigma^{-1}$ is constructed by inverting the non-zero diagonal entries. 
Let $\row(A)$ and $\col(A)$ be the rowspace of $A$ and the column space respectively. 

Throughout the paper, we assume our data $A \in \mathbbm{R}^{n \times d}$ is sampled from an unknown distribution $\mathcal{D}$. Specifically, we have a training set $\Tr = \left\{A_1, \ldots, A_N\right\} \in \mathcal{D}$.
The generic form of our optimization problems is $\min_X f(A, X)$ where $A$ is the input matrix and $X$ is a feasible solution of $f$ on $A$. 

For a given optimization problem and a set of sketching matrices $\sS$, define algorithm $\SKALG(\sS, A)$ as its classical sketching algorithm; this uses the sketching matrices in $\sS$ to map the given input $A$ and construct an approximate solution $\hat{X}$. We remark that the number of sketches used by an algorithm can vary and in its simplest case, $\sS$ is a single sketches but in more complicated sketching approaches we may need to apply sketching more than one---hence $\sS$ may have more than one sketch matrix.

\paragraph{CountSketch.} We define $S_C \in \mathbbm{R}^{m \times n}$ as a classical CountSketch (abbreviated as $\CS$) matrix. It is a sparse matrix with one nonzero entry from $\{\pm 1\}$ per column. The position and value of this nonzero entry is chosen uniformly at random. CountSketch matrices can be succinctly represented as two vectors. We define $p\in [m]^n, v \in \mathbbm{R}^{n}$ as the positions and values of the nonzero entries. Further, we define $\CS(p, v)$ as the CountSketch constructed from vectors $p$ and $v$. 

\paragraph{Problem Definitions.}
Low-rank decomposition (LRA), and $k$-means clustering are both problems of the form $\min_X f(A, X)$. In this section, we define the objective function $f(\cdot)$ and a time-optimal classical sketching algorithm $\SKALG(\sS, A)$ for each problem. 
The classical sketching algorithms work with {\em classical sketches} and provide the worst-case guarantee of form:
\begin{center}
  \emph{For runtime $\mathcal{T}(\epsilon)$, w.h.p, $\Delta := f(A, \SKALG(\sS, A)) - f(A, X^{*}) \leq \epsilon f(A, X^{*}) ,\; \forall A$,}  
\end{center}
where $X^* = \argmin_{X} f(A,X)$. Note that since the dimension of the sketch matrix (or sketch matrices) $m$ fully determines the runtime, the worst-case guarantee can also be stated in terms of sketch matrix dimension $m$ instead of $\mathcal{T}(\epsilon)$.

\begin{itemize}[leftmargin=*]
    \item{\bf Low-rank approximation (LRA).} In LRA, we find a $\rankk$ approximation of our data that minimizes the norm of the approximation error. For $A\in \mathbb{R}^{n\times d}$,
    \begin{align*}
        \min_{\rankk X} f_{\LRA}(A,X) = \min_{\rankk X} \norm{A - X}_F^2
    \end{align*}
    \item{\bf $k$-means clustering.} In $k$-means clustering, we find $k$ clusters on our data points so as to minimize the sum of squared distances of each point to its cluster center. Given points $A_{1},\ldots,A_n \in \mathbb{R}^d$ (succinctly represented as $A \in \mathbbm{R}^{n \times d}$), we find $k$ clusters $\C=\{\C_1,\ldots,\C_k\}$ with centers $\{\mu_1,\ldots,\mu_k\}$:
    \begin{align*}
        \min_{\C} f_{k\text{-means}}(A, \C) = \min_{\C} \sum_{i\in [k]}\min_{\mu_i\in\mathbb{R}^d}\sum_{j \in \C_i} \left\|{A_j -\mu_i}\right\|_2^2
    \end{align*}
\end{itemize}

\paragraph{Learned CountSketch Paradigm of~\cite{indyk2019learning}.}
Consider a scenario in which a sketching algorithm is run on inputs from a distribution $\mathcal{D}$. We have training samples from this distribution: $\Tr = \left\{A_1, \ldots, A_N\right\} \in \mathcal{D}$. 

The learned sketch framework has two parts: (1) offline sketch learning and (2) ``online'' sketching (i.e. applying the learned sketch and some sketching algorithm to possibly unseen data). 
In offline sketch learning, the goal is to construct a CountSketch matrix with the minimum expected error for the problem of interest. Formally, that is,
\begin{align*}
\min_{\CS~S} \E_{A\in \Tr} f(A, \SKALG(S, A)) - f(A, X^{*})
= \min_{\CS~S} \E_{A\in \Tr} f(A, \SKALG(S, A))
\end{align*}
where $X^{*}$ denotes the optimal solution. Moreover, the minimum is taken over all possible constructions of CountSketch. We remark that when $\SKALG$ needs more than one CountSketch to be learned (e.g., in the sketching algorithm we consider for LRA), we optimize each CountSketch independently using a surrogate loss function. 

In the second part of the learned sketch paradigm, we take the sketch from part one and use it within a sketching algorithm. This learned sketch and sketching algorithm can be applied, again and again, to different inputs. Finally, we augment the sketching algorithm to provide worst-case guarantees when used with learned sketches. The goal is to have good performance on $A \in \mathcal{D}$ while the worst-case performance on $A \not\in \mathcal{D}$ remains comparable to the guarantees of classical sketches.

\begin{remark}
When we mention time complexity in this paper, we refer to the runtime of online sketching in the learned sketch paradigm. We do not include the runtime of offline sketch learning; this is a one-time cost that will be amortized against the many applications of the sketching algorithm. However, for completeness and fair comparison, we do provide wall-clock timing for the offline stage in all our experiments in Section~\ref{sec:evaluation}.
\end{remark}
\section{High-Level Description of Our Approach}
We describe our contributions to the learning-based sketching paradigm which, as mentioned, is to {\em learn the locations of non-zero values} in the sketch matrix. In Section~\ref{sec:local-opt}, we provide provable guarantees for the performance of this method on two natural families of distributions.

Moreover, to attain similar worst-case guarantees to the ones of sketching algorithms with the classical CountSketch, in Section~\ref{sec:approxcheck}, we consider a simple ``approximately compare the solution and return the better of the two'' approach denoted as {\em ApproxCheck} for all problems studied in this paper. 
We emphasize that this simple approach is asymptotically is as good as the sketch monotonicity approach of~\citeauthor{indyk2019learning}. In Section~\ref{sec:sketch-monotonicity}, 
We show that the sketch monotonicity property holds for the time-optimal sketching algorithm of LRA and the sketching algorithm of $k$-means. 

\subsection{Learning CountSketch}\label{sec:learn-countsketch}
To learn a CountSketch for the given training data set, We will locally optimize 
\begin{align}
\min_{S} \E_{A \in \mathcal{D}} \left[ f(A, \SKALG(S, A)) \right]
\label{eq:obj_fn}
\end{align}
in two stages: (1) we compute the positions of the nonzero entries, then (2) we fix the positions and optimize their values. 

\paragraph{Stage 1: Optimizing Positions.}  
In sparse sketches, the placement of no-zero values has a huge effect on the compression quality.  
Consider the time-optimal sketching algorithm for LRA (Algorithm~\ref{alg:lowrank-sketch}) which seeks the best $\rankk$ approximation to $A$ in $\row(SA) \cap \col(AR)$. The optimal $\rankk$ approximation is $A_k = U_k \Sigma_k V_k^{\top}$, which belongs to $\col(U_k) \cap \row(V_k^{\top})$. Thus, we get a better approximate solution if $\row(SA)$ closely captures $\row(V_k^{\top})$, the $k$ most important components of $A$'s row space (respectively, if $\col(AR)$ closely captures $\col(U_k)$ too). Now, when $S$ is CountSketch, computing $SA$ amounts to hashing the $n$ rows of $A$ to the $m \ll n$ rows of $SA$. In order to preserve the $k$ most important components of $A$'s row space, the sketch should not hash the ``important'' rows of $A$ to the same bin. 
In fact, this insight is the basis of our theoretical contributions for optimizing positions, Theorem~\ref{thm:greedy_general}, which says that for LRA and certain families of input matrices, position optimization is \textit{provably} better than the classical CountSketch. 

\begin{algorithm}[H]
	\begin{algorithmic}[1]
		\REQUIRE{$A \in \mathbb{R}^{n \times d}, S\in \mathbb{R}^{m_S \times n}, R\in \mathbb{R}^{m_R \times d}, V \in \mathbb{R}^{m_V \times n}, W \in \mathbb{R}^{m_W \times d}$} 	
		\STATE $U_C \begin{bmatrix} T_C & T_C^{'} \end{bmatrix}\leftarrow V A R^{\top}, \begin{bmatrix} T_D^{\top} \\ T_D^{' \top} \end{bmatrix} U^\top_D \leftarrow S A W^\top$ with $U_C, U_D$ orthogonal 
		\STATE $G\leftarrow V A W^\top$, \quad $Z_L^{'} Z_R^{'} \leftarrow [U^\top_C G U_D]_k$
		\STATE $Z_L = \begin{bmatrix} Z_L^{'} (T_D^{-1})^{\top} & 0 \end{bmatrix}, Z_R = \begin{bmatrix} T_C^{-1} Z_R^{'} \\ 0 \end{bmatrix}$
		\STATE $Z = Z_L Z_R$
		\STATE {\bfseries return:} $AR^\top Z SA$ in form $P_{n\times k}, Q_{k\times d}$ 
	\end{algorithmic}
	\caption{$\SKALG_{\LRA}$(\textsc{Sketch-LowRank})~\citep{sarlos2006improved,clarksonwoodruff,avron2016sharper}.}
	\label{alg:lowrank-sketch}
\end{algorithm}

Now, we explain how to do this optimization. This is a combinatorial optimization problem with an empirical risk minimization (ERM) objective. The na\"ive solution is to compute the objective value of the exponentially many ($m^n$) possible placements, but this is clearly intractable. Instead, we iteratively construct a full placement in a greedy fashion. We start with $S$ as a zero matrix. Then, we iterate through the columns of $S$ in an order determined by the algorithm, adding a nonzero entry to each. The best position in each column is the one that minimizes Eq.~\eqref{eq:obj_fn} if an entry were to be added there. 
For each column, we evaluate Eq.~\eqref{eq:obj_fn} $\mathcal{O}(m)$ times, once for each prospective half-built sketch. In the following algorithm, $e_i$ is the $i^{th}$ standard unit vector and $(e_{j} e_i^{\top})$ is a matrix with $1$ at entry $(j, i)$ and zeros everywhere else. 
\begin{algorithm}[!ht]
	\begin{algorithmic}[1]
		\REQUIRE{$f, \SKALG, \Tr = \{A_1,...,A_{N} \in \mathbb{R}^{n \times d}\}$; sketch dimension $m$} 	
		\STATE {\bf initialize} $S_L = \mathbbm{O}^{m \times n}$
		\FOR{$i = 1$ to $n$}
			\STATE $\bar{j} = \underset{j \in [m]}{\argmin} \underset{A \in \Tr}{\sum}  f(A, \SKALG(S_L \pm (e_j e_i^{\top}), A))$ 
			\STATE $S_L = S_L \pm (e_{\bar{j}} e_i^{\top})$
		\ENDFOR
		\STATE \textbf{return} $p$ for $S_L = \CS(p, v)$
	\end{algorithmic}
	\caption{\textsc{Position optimization}}
	\label{alg:greedy_alg}
\end{algorithm}
While this greedy strategy is simple to state, additional tactics are required for each problem to make it more tractable. For LRA and $k$-means, the objective evaluation (Algorithm~\ref{alg:greedy_alg}, line 3) is too slow, so we must leverage our insight into their sketching algorithms to pick a proxy objective. Note that we can reuse these proxies for value optimization, since they may make gradient computation faster too. 
\begin{itemize}[leftmargin=*]
    \item\textbf{Proxy objective for LRA.} For LRA, we optimize the positions in both sketches $S$ and $R$. We cannot use $f(A, \SKALG(S, R, A))$ as our objective because then we would have to consider \textit{combinations} of placements between $S$ and $R$. To find a proxy, we note that a prerequisite for good performance is for $\row(SA)$ and $\col(AR)$ to both contain a good $\rankk$ approximation to $A$ (see proof of Lemma~\ref{lem:two-sketch}). Thus, we can decouple the optimization of $S$ and $R$. The proxy objective for $S$ is $\norm{[AV]_k V^{\top} - A}_F^2$ where $SA = U\Sigma V^{\top}$. In this expression, $\hat{X} = [AV]_k V^{\top}$ is the best $\rankk$ of approximation of $A$ in $\row(SA)$. The proxy objective for $R$ is defined analogously.
    
    \item\textbf{Proxy objective for $k$-means.} 
    To find a proxy objective, we use the fact that we have good performance if $A$ is projected to an approximate top singular vector space~\citep{cohen2015dimensionality}. That is, if $SA =  U \Sigma V^{\top}$ and $A_m = U_m \Sigma_m V_m^{\top}$ for sufficiently large $m$, we want $\row(V^{\top}) \sim \row(V_m^{ \top})$. Now, $V_m$ is the minimizer of $\min_{V \in \mathbbm{R}^{d \times m}} \norm{[AV]_k V^{\top} - A}_F^2$, so the better $V$ approximates this expression, the more ``similar'' $\row(V^{\top})$ and $\row(V_m^{\top})$ are. Thus, our proxy for $k$-means is $\norm{[AV]_k V^{\top} - A}_F^2$, which is the same as for LRA. 
\end{itemize}

\paragraph{Stage 2: Optimizing Values.} This stage is similar to the approach of~\cite{indyk2019learning}. However, instead of the power method, we use an automatic differentiation package, PyTorch~\citep{pytorch}, and we pass it our objective, Eq.~\eqref{eqn:grad_objective}, implemented as a chain of differentiable operations. It will automatically compute the gradient using the chain rule.
\begin{align}
\min_{v \in \mathbbm{R}^{n}} \E_{A \in \mathcal{D}} \left [ f(A, \SKALG(\CS(p, v), A)) \right] \label{eqn:grad_objective} 
\end{align}

\subsection{Provable Guarantees for Location Optimization}\label{sec:local-opt}
While the position optimization idea is simple, one particularly interesting aspect is that it is provably better than a random placement in some scenarios (Theorem.~\ref{thm:greedy_general}). Specifically, it is provably beneficial for LRA when inputs follow the spiked covariance or Zipfian distribution (relatively common for real data) and we iterate through the columns of $S$ in a certain order. 

\begin{itemize}[leftmargin=*]
    \item{\bf Spiked covariance model with sparse left singular vectors.} Every matrix $A \in \mathbb{R}^{n \times d}$ from the distribution $\sA_{sp}(s,\ell)$ has $s < k$ ``heavy'' rows ($A_{r_1}, \cdots, A_{r_s}$) of norm $\ell >1$. The indices of the heavy rows can be arbitrary, but must be the same for all members of the distribution and are unknown to the algorithm.  The remaining rows (called ``light'' rows) have unit norm. 
    In other words, let $\mathcal{R} = \{r_1, \ldots, r_s\}$. For all rows $A_i, i \in [n]$, $A_i = \ell \cdot v_i$ if $i\in \mathcal{R}$ and $A_i = v_i$ otherwise, where $v_i$ is a uniformly random unit vector.
    \item{\bf Zipfian on squared row norms.} Every matrix $A \in \mathbb{R}^{n \times d} \sim \sA_{zipf}$ has rows which are uniformly random and orthogonal. Each $A$ has $2^{i+1}$ rows of squared norm $n^2/2^{2i}$ for $i \in [1, \ldots, \O(\log(n))]$. We also assume that each row has the same squared norm for all members of $\sA_{zipf}$. 
\end{itemize}

\begin{theorem}\label{thm:greedy_general}
Consider a matrix $A$ from either the spiked covariance or Zipfian distributions. Let $S_L$ denote a CountSketch constructed by Algorithm~\ref{alg:greedy_alg} that optimizes the positions of non-zero values w.r.t.~$A$. Let $S_C$ denote a CountSketch matrix. Then there is a fixed $\eta > 0$ s.t., 
\begin{align*}
\min_{\rankk X \in \rowsp(S_L A)}\left\| X - A \right\|_F^2 \leq (1 - \eta) \min_{\rankk X \in \rowsp(S_C A)}\left\| X - A \right\|_F^2. 
\end{align*}
\end{theorem}
\begin{remark}
Note that the above theorem implicitly provide an upper bound on the generalization error of the greedy placement method on the two distributions tha we considered in this paper. More precisely, for each of these two distributions, if the learned sketch $\Pi$ is learned via our greedy approach over a set of training matrices sampled from the distributions, the solution returned by the sketching algorithm using $\Pi$ over any (test) matrix $A$ sampled from the distribution has error at most $(1 - \eta) \min_{\rankk X \in \rowsp(S_C A)}\left\| X - A \right\|_F^2$.  
\end{remark}

A key structural property of the matrices from these two distributions that is crucial in our analysis is the {\em $\eps$-almost orthogonality} of their rows (i.e., (normalized) pairwise  inner products are at most $\eps$). Hence, we can find a QR-factorization of the matrix of such vectors where the upper diagonal matrix $R$ has diagonal entries close to $1$ and entries above the diagonal are close to $0$.  

To state our result, we first provide an interpretation of the location optimization task as a selection of hash function for the rows of $A$. Note that left-multiplying $A$ by CountSketch $S \in \mathbb{R}^{m \times n}$ is equivalent to hashing the rows of $A$ to $m$ bins with coefficients in $\{\pm 1\}$. The greedy algorithm proceeds through the rows of $A$ (in some order) and decides which bin to hash to, denoting this by adding an entry to $S$. 
The intuition is that our greedy approach separates heavy-norm rows (which are important ``directions'' in the row space) into different bins.  

{\em Proof outline:} The first step is to observe that in the greedy algorithm, when rows are examined according to a non-decreasing order of squared norms, will isolate rows into their singleton bins until all bins are filled. In particular, this means that the heavy norm rows (the first to be processed) will all be isolated. 

Next, we show that none of the rows left to be processed (all light rows) will be assigned to the same bin as a heavy row. The main proof idea is to compare the cost of ``colliding'' with a heavy row to the cost of ``avoiding'' the heavy rows. Specifically, we compare the \textit{decrease} (before and after bin assignment of a light row) in the sum of squared projection coefficients, lower-bounding it in the former case and upper-bounding it in the latter. This is the main place we use the properties of the aforementioned distributions and the fact that each heavy row is already mapped to a singleton bin.  
Overall, we show that at the end of the algorithm no light row will be assigned to the bins that contain heavy rows. The complete proof of Theorem~\ref{thm:greedy_general} is deferred to Section~\ref{sec:greedy-init}.

\subsection{Attaining Worst-Case Guarantee: Sketch Monotonicity}\label{sec:sketch-monotonicity}
Both the sketch concatenation method (a.k.a. MixedSketch)---whose guarantee is via the sketch monotonicity property---and approximate comparison method (a.k.a. ApproxCheck), which just evaluates the cost of two solutions and takes the better one, asymptotically achieve the same worst-case guarantee. However, for any input matrix $A$ and any pair of sketches $S, T$, the performance of the MixedSketch method on $(A, S, T)$ is never worse than the performance of its corresponding ApproxCheck method on $(A, S, T)$, and can be much better. 
\begin{remark}
Let $A = \mathrm{diag}(2, 2, \sqrt{2}, \sqrt{2})$, and suppose the goal is to find a rank-$2$ approximation of $A$. Consider two sketches $S$ and $T$ such that $SA$ and $TA$ capture $\Span(e_1, e_3)$ and $\Span(e_2, e_4)$ respectively. 
Then for both $SA$ and $TA$, the best solution in the subspace of one of these two spaces is a $(\frac{3}{2})$-approximation: $\norm{A - A_2}_F^2 = 4$ and $\norm{A - P_{SA}}_F^2 = \norm{A - P_{TA}}_F^2 = 6$ where $P_{SA}$ and $P_{TA}$ respectively denote the best approximation of $A$ in the space spanned by $SA$ and $TA$. 

However if we find the best rank-$2$ approximation of $A$, $Z$, inside the span of the union of $SA$ and $TA$, then $\norm{A - Z}_F^2 = 4$. Since ApproxCheck just chooses the better of $SA$ and $TA$ by evaluating their costs, it misses out on the opportunity to do as well as MixedSketch.
\end{remark}

In this section, we prove the sketch monotonicity property for (time-optimal) sketching algorithms of LRA and $k$-means. Previously,~\citeauthor{indyk2019learning} had only proved it for a na\"ive sketching algorithm of LRA and never for $k$-means. 

\paragraph{Sketch Monotonicity for LRA.} Here, we show the sketch monotonicity property for LRA.
\begin{theorem}\label{thm:mixedsketch-LRD}
Let $A \in \mathbbm{R}^{n \times d}$ be an input matrix, $V$ and $W$ be $\eta$-affine embeddings, and $S_1 \in \mathbbm{R}^{m_S \times n}, R_1 \in \mathbbm{R}^{m_R \times n}$ be arbitrary matrices. Consider arbitrary extensions to $S_1, R_1$: $\bar{S}, \bar{R}$ (e.g., $\bar{S}$ is a concatenation of $S_1$ with an arbitrary matrix with the same number of columns). Then,  
    $\norm{A - \SKALG_\LRA((\bar{S}, \bar{R}, V, W), A))}_F^2
    \leq (1 + \eta)^2 \norm{A - \SKALG_\LRA((S_1, R_1, V, W), A)}_F^2$
\end{theorem}
\begin{proof}
We have
\begin{align*}
    \norm{A - \SKALG_\LRA((\bar{S}, \bar{R}, V, W), A)}_F^2 
    &\leq (1 + \eta) \min_{\rankk X} \norm{A\bar{R} X \bar{S}A - A}_F^2 \\
    &= (1 + \eta) \min_{\rankk X: X\in \row(\bar{S}A) \cap \col(A\bar{R})} \norm{X - A}_F^2,
\end{align*}
which is in turn at most
\begin{align*}
 (1 + \eta) \min_{\rankk X: X\in \row(S_1A) \cap \col(AR_1)} \norm{X - A}_F^2 
 &= (1 + \eta) \min_{\rankk X} \norm{AR_1 X S_1A - A}_F^2\\ 
 &\leq (1 + \eta)^2 \norm{A - \SKALG_\LRA((S_1, R_1, V, W), A)}_F^2,
\end{align*}
where we use the fact the $V, W$ are affine $\eta$-embeddings (Definition~\ref{def:affine-embedding}), 
as well as the fact that $\left( \col(AR_1) \cap \row(S_1A) \right) \subseteq \left( \col(A\bar{R}) \cap \row(\bar{S}A) \right)$. 
\end{proof}
\paragraph{Sketch Monotonicity for $k$-means.} Next, we prove the sketch monotonicity property for $k$-means.
\begin{algorithm}[!h]
	\begin{algorithmic}[1]
    \REQUIRE{$A \in \mathbbm{R}^{n \times d}, S \in \mathbb{R}^{m \times n}$, $\alpha$-approximation algorithm of $k$-means: $\sA$}
    \STATE \textbf{return} $\sA(AV)$ \COMMENT{where $SA = U \Sigma V^{\top}$} 
	\end{algorithmic}
	\caption{\textsc{sketch-kmeans}~\citep{cohen2015dimensionality}}
	\label{alg:kmeans-sketch}
\end{algorithm}
\begin{theorem}\label{thm:mixedsketch-kmeans}
Let $A\in \mathbb{R}^{n \times d}$ be an input matrix and let $S \in \mathbb{R}^{m \times n}$ be a sketch matrix satisfying $\norm{A - AVV^{\top}}_2^2 \leq \frac{\eta}{k}\norm{A - A_k}_F^2$ for $SA = U \Sigma V^{\top}$. 
Let $\bar{S}$ be an arbitrary extension to $S$. Then, the solution computed via Algorithm~\ref{alg:kmeans-sketch} using $\bar{S}$ and $\alpha$-approximation algorithm of $k$-means is an $\left(\alpha\cdot(1 + \O(\eta)) \right)$-approximation.
\end{theorem}
\begin{proof}
We define $\C^{*} = \{C^*_1,\cdots, C^*_k\}, \C = \{C_1,\cdots, C_k\}$ as the optimal and $\alpha$-approximate clusterings on the points projected to $SA$. Also, let $\mu_{C^*_1}, \cdots, \mu_{C^*_k},\mu_{C_1}, \cdots, \mu_{C_k} \in \row(SA)$ be their corresponding cluster centers. We define clusters $\bar{\C}^{*} = \{\bar{C}^*_1,\cdots, \bar{C}^*_k\}, \bar{\C} = \{\bar{C}_1,\cdots, \bar{C}_k\}$ and centers $\mu_{\bar{C}^*_1}, \cdots, \mu_{\bar{C}^*_k}, \mu_{\bar{C}_1}, \cdots, \mu_{\bar{C}_k}$ analogously for points projected onto $\bar{S}A$. Let $\pi(\cdot)$ be the projection operator, with its subscript specifying the subspace to project to. Finally, let $X^{*}$ be the optimal clustering of $A$. 
\begin{align}
& \sum_{i\in [k]} \min_{\mu_i} \sum_{j \in \bar{C}_i} \norm{A_j - \mu_i}^2 \nonumber\\
&\leq \sum_{i\in [k]} \sum_{j \in \bar{C}_i} \norm{A_j - \mu_{\bar{C}_i}}^2 \nonumber\\
&= \sum_{i\in [k]} \sum_{j \in \bar{C}_i} \left[ \norm{A_j - \pi_{\row(\bar{S}A)}(A_j)}^2 + \norm{\pi_{\row(\bar{S}A)}(A_j) - \mu_{\bar{C}_i}}^2 \right] \label{thm:eq1}\\
&\leq \sum_{j\in [n]} \norm{A_j - \pi_{\row(\bar{S}A)}(A_j)}^2  +
\alpha \sum_{i\in [k]} \sum_{j \in \bar{C}^{*}_i} \norm{\pi_{\row(\bar{S}A)}(A_j) - \mu_{\bar{C}^{*}_i}}^2  \label{thm:eq2} \\
&\leq \sum_{j\in [n]} \norm{A_j - \pi_{\row(\bar{S}A)}(A_j)}^2 +
\alpha \sum_{i\in [k]}\sum_{j \in C^{*}_i} \norm{\pi_{\row(\bar{S}A)}(A_j) - \mu_{C^{*}_i}}^2 \label{thm:eq3} \\
&\leq \alpha \sum_{i\in [k]}\sum_{j \in C^{*}_i} \left[ \norm{A_j - \pi_{\row(\bar{S}A)}(A_j)}^2 + \norm{\pi_{\row(\bar{S}A)}(A_j) - \mu_{C^{*}_i}}^2 \right] \nonumber\\
&= \alpha \sum_{i\in [k]} \sum_{j \in C^{*}_i} \norm{A_j - \mu_{C^{*}_i}}^2 \label{thm:eq4} \\ 
&\leq \alpha (1 + \eta) \sum_{i\in [k]} \min_{\mu_i} \sum_{j \in C^{*}_i} \norm{A_j - \mu_i}^2 \label{thm:eq5}\\
&\leq \alpha (1 + \eta)^2 f_{k\text{-means}}(A, X^{*}), \label{thm:eq6}
\end{align}
where~\eqref{thm:eq1} and~\eqref{thm:eq4} hold by Pythagorean Theorem since $\mu_{\bar{C}_i} \in \row(\bar{S}A)$ and $\mu_{C^{*}_i} \in \row(SA) \subset \row(\bar{S}A)$.
\eqref{thm:eq2} holds since $\bar{\C}$ is an $\alpha$-approximate $k$-means of $SA$ and $\bar{\C}^{*}$ is an optimal $k$-means of $SA$.
\eqref{thm:eq3} holds since $\bar{\C}^{*}$ is an optimal clustering for points projected to $\bar{S}A$.
\eqref{thm:eq5} holds by Corollary~\ref{cor:diff_centers}. Finally,~\eqref{thm:eq6} holds by our assumption on $S$, which implies $\norm{A - \hat{A}_k}_F^2 \leq (1 + \eta) \norm{A - A_k}_F^2$ for some $\hat{A}_k \in \row(SA)$ by Lemma~\ref{lem:assumption_implies_approx_lrd}. This fact allows us to apply Theorem 9 of~\cite{cohen2015dimensionality}, which implies that the $\C^{*}$ is a $\left( 1 + \eta \right)$-approximation of $X^{*}$.  
\end{proof}

\begin{remark}
Our result shows that the ``sketch monotonicity'' property holds for  sketching matrices that provide strong coresets for $k$-means clustering. 
Besides strong coresets, an alternate approach to showing that the clustering objective is approximately preserved on sketched inputs is to show a weaker property: the clustering cost is preserved for {\em all possible partitions} of the points into $k$ groups~\citep{MakarychevMR19}. While the dimension reduction mappings satisfying strong coresets require $\poly(k/\eps)$ dimensions,~\cite{MakarychevMR19} shows that $O(\log k /\eps^2)$ dimensions suffice to satisfy this ``partition'' guarantee. An interesting question for further research is if the sketch monotonicity guarantee also applies to the construction of~\cite{MakarychevMR19}.   
\end{remark}
\section{Empirical Evaluation}\label{sec:evaluation}
A main advantage of our approach is that it excels on natural data distributions. In this section, we empirically verify this by testing our learning-based sketching approach on five datasets. For each of LRA and $k$-means, we fix runtime (sketch size) and compare approximation error ($\Delta = f(A, \SKALG(\sS, A)) - f(A, X^{*})$) averaged over 10 trials.  
We implemented everything in PyTorch, which allowed us to harness GPUs for fast, parallel computation on our large data sets. In order to make position optimization more efficient, we also leveraged a trick from numerical linear algebra. We noted for LRA that placing each entry (line 3 in Algorithm~\ref{alg:greedy_alg}) involved computing many rank-$1$ SVD updates. Instead of doing this na\"{\i}vely, we used some formulas for \textit{fast} rank-$1$ SVD updates~\citep{fastsvd}, which improved the runtime greatly. We used several Nvidia GeForce GTX 1080 Ti machines for all our experiments. 

\paragraph{Sketch Matrices.} We consider different constructions of CountSketch with different degrees of optimization: 
\textbf{Ours} (learned positions and values), \textbf{IVY19} (random positions, learned values), and \textbf{Classical CS} (random positions and values). We also included two na\"ively ``learned'' sketches to see if optimizing on just one entry of $\Tr$ is sufficient: \textbf{Exact SVD} (compresses $A$ to $AV_m$, where $V_m$ contains the top $m$ right singular vectors of a single sample from $\Tr$) and \textbf{Column sampling} (compresses $A$ to $AP^{\top}$, where $P$ is a CS and built from a single sample from $\Tr$. The position of each non-zero entry is sampled proportional to the row norms; the value is the norm's inverse).
\begin{table}[t!]
	\centering
	\caption{Data set descriptions}		
	\label{tab:dataset_descriptions}
	\resizebox{\textwidth}{!}{	
		\renewcommand{\arraystretch}{1}
\begin{tabular}{|l|l|p{0.35\textwidth} |l|l|l|}
\hline
\textbf{Name} & \textbf{Domain} & \textbf{Description}  & \textbf{$A$ shape} & \textbf{$N_{\mathrm{train}}$} & \textbf{$N_{\mathrm{test}}$} \\ \hline
Friends & Image & Frames from a scene in the TV show \textit{Friends}\tablefootnote{\url{http://youtu.be/xmLZsEfXEgE}} & $5760\times 1080$ &  $400$ & $100$ \\ \hline
Logo & Image & Frames from video of logo being painted\tablefootnote{\url{http://youtu.be/L5HQoFIaT4I}} & $5760\times 1080$ & $400$ & $100$ \\ \hline

Hyper & Image & Hyperspectral images depicting outdoor scenes ~\citep{imamoglu2018} & $1024\times 768$ & $400$ & $100$ \\ \hline

Yelp & Natural language & tf-idf ~\citep{ramos2003} of restaurant reviews, grouped by restaurant\tablefootnote{\url{https://www.yelp.com/dataset}}& $7000 \times 800$ & $260$ & $65$  \\ \hline

Social & Graph & Graph adjacency matrices of social circles on Google+ ~\citep{leskovec2012} & $1052 \times 1052$ & $147$ & $37$ \\ \hline
\end{tabular}
}
\end{table}
\paragraph{Data Sets.} 
We used five high-dimensional data sets from three different domains (image, natural language, and graph). We offer a brief description here, but additional details, (data dimension, $N_{\mathrm{train}}$, etc.) can be found in~Table~\ref{tab:dataset_descriptions}. Our data sets are: (1, 2) \textbf{Friends, Logo} (image): frames from videos of a scene from the TV show \textit{Friends} and of a logo being painted; (3) \textbf{Hyper} (image): hyperspectral images of outdoor environments; (4) \textbf{Yelp} (natural language): term-document matrices\footnote{These contain word counts for a group of text documents. Rows represent unique words; columns represent documents.} of restaurant reviews, with documents grouped by restaurant; (5) \textbf{Social} (graph): graph adjacency matrices of different social circles on Google+. Another argument for the practicality of our approach is that some of these experiments have real-world counterparts. For example, LRA is frequently applied to term-document matrices in ``latent semantic analysis'', a natural language processing technique for relating words and categorizing documents, among many other things~\citep{latentsemantic}. 

\paragraph{Results Summary.} 
For LRA, \textbf{Ours} is always the best of the five sketches. It is significantly better than \textbf{Classical CS}, obtaining improvements of around 70\% for LRA. For $k$-means, ``Column Sampling'' and ``Exact SVD'' dominated on several data sets. However, we note that our method was always a close second and also, more importantly, the values for $k$-means were only trivially different ($<1\%$) between methods. Finally, we note that including the position optimization is strictly better in all cases (compare to \textbf{IVY19}). For LRA, , it offers around 20\%improvement.
\begin{table*}[h!]
	\centering
	\caption{Average errors for LRA}		
	\label{tab:lra-4}
	\resizebox{\textwidth}{!}{	
		\renewcommand{\arraystretch}{1.15}
\begin{tabular}{|l|l|l|l|l|l|l|}
\hline
\textbf{Parameters}       & \textbf{Algorithm}  & \multicolumn{5}{c|}{\textbf{Datasets}}            \\ \hline
                          \textbf{rank $k$, m} & & \textbf{Friends} & \textbf{Hyper} & \textbf{Logo} & \textbf{Yelp} & \textbf{Social} \\ \hline
(20, 40) & Ours &  \textbf{0.8998} \textcolor{mygray}{$\pm 2.7\mathrm{e}{-02}$} &  \textbf{2.4977} \textcolor{mygray}{$\pm 1.8\mathrm{e}{-01}$} &  \textbf{0.5009} \textcolor{mygray}{$\pm 2.2\mathrm{e}{-02}$} &  \textbf{0.1302} \textcolor{mygray}{$\pm 3.5\mathrm{e}{-04}$} &  \textbf{30.0969} \textcolor{mygray}{$\pm 1.3\mathrm{e}{-01}$} \\ \cline{2-7} 
 & IVY19 &  1.0484 \textcolor{mygray}{$\pm 1.3\mathrm{e}{-02}$} &  3.7648 \textcolor{mygray}{$\pm 4.2\mathrm{e}{-02}$} &  0.6879 \textcolor{mygray}{$\pm 8.8\mathrm{e}{-03}$} &  0.1316 \textcolor{mygray}{$\pm 8.8\mathrm{e}{-04}$} &  30.5237 \textcolor{mygray}{$\pm 1.5\mathrm{e}{-01}$} \\ \cline{2-7} 
 & Classical CS &  4.0730 \textcolor{mygray}{$\pm 1.7\mathrm{e}{-01}$} &  6.3445 \textcolor{mygray}{$\pm 1.8\mathrm{e}{-01}$} &  2.3721 \textcolor{mygray}{$\pm 8.3\mathrm{e}{-02}$} &  0.1661 \textcolor{mygray}{$\pm 1.4\mathrm{e}{-03}$} &  33.0651 \textcolor{mygray}{$\pm 3.4\mathrm{e}{-01}$} \\ \cline{2-7} 
 & Exact SVD &  1.5774 \textcolor{mygray}{$\pm 1.1\mathrm{e}{-01}$} &  3.4406 \textcolor{mygray}{$\pm 8.7\mathrm{e}{-01}$} &  0.7470 \textcolor{mygray}{$\pm 1.0\mathrm{e}{-01}$} &  0.1594 \textcolor{mygray}{$\pm 3.7\mathrm{e}{-03}$} &  31.0617 \textcolor{mygray}{$\pm 3.4\mathrm{e}{-01}$} \\ \cline{2-7} 
 & Col. samp. &  5.9837 \textcolor{mygray}{$\pm 6.6\mathrm{e}{-01}$} &  9.7126 \textcolor{mygray}{$\pm 8.2\mathrm{e}{-01}$} &  4.2008 \textcolor{mygray}{$\pm 6.0\mathrm{e}{-01}$} &  0.1881 \textcolor{mygray}{$\pm 3.2\mathrm{e}{-03}$} &  43.9920 \textcolor{mygray}{$\pm 5.6\mathrm{e}{-01}$} \\ \hline 
(30, 60) & Ours &  \textbf{0.7945} \textcolor{mygray}{$\pm 1.5\mathrm{e}{-02}$} &  \textbf{2.4920} \textcolor{mygray}{$\pm 2.6\mathrm{e}{-01}$} &  \textbf{0.4929} \textcolor{mygray}{$\pm 2.3\mathrm{e}{-02}$} &  \textbf{0.1128} \textcolor{mygray}{$\pm 3.2\mathrm{e}{-04}$} &  \textbf{30.6163} \textcolor{mygray}{$\pm 1.8\mathrm{e}{-01}$} \\ \cline{2-7} 
 & IVY19 &  1.0772 \textcolor{mygray}{$\pm 1.2\mathrm{e}{-02}$} &  3.7488 \textcolor{mygray}{$\pm 1.8\mathrm{e}{-02}$} &  0.7348 \textcolor{mygray}{$\pm 7.1\mathrm{e}{-03}$} &  0.1137 \textcolor{mygray}{$\pm 5.9\mathrm{e}{-04}$} &  30.9279 \textcolor{mygray}{$\pm 2.0\mathrm{e}{-01}$} \\ \cline{2-7} 
 & Classical CS &  2.6836 \textcolor{mygray}{$\pm 7.4\mathrm{e}{-02}$} &  5.3904 \textcolor{mygray}{$\pm 7.7\mathrm{e}{-02}$} &  1.6428 \textcolor{mygray}{$\pm 4.3\mathrm{e}{-02}$} &  0.1463 \textcolor{mygray}{$\pm 2.3\mathrm{e}{-03}$} &  32.7905 \textcolor{mygray}{$\pm 2.1\mathrm{e}{-01}$} \\ \cline{2-7} 
 & Exact SVD &  1.1678 \textcolor{mygray}{$\pm 5.2\mathrm{e}{-02}$} &  3.0648 \textcolor{mygray}{$\pm 8.5\mathrm{e}{-01}$} &  0.5958 \textcolor{mygray}{$\pm 7.2\mathrm{e}{-02}$} &  0.1326 \textcolor{mygray}{$\pm 2.3\mathrm{e}{-03}$} &  32.0861 \textcolor{mygray}{$\pm 5.9\mathrm{e}{-01}$} \\ \cline{2-7} 
 & Col. samp. &  4.1899 \textcolor{mygray}{$\pm 2.9\mathrm{e}{-01}$} &  8.2314 \textcolor{mygray}{$\pm 4.7\mathrm{e}{-01}$} &  3.2296 \textcolor{mygray}{$\pm 3.1\mathrm{e}{-01}$} &  0.1581 \textcolor{mygray}{$\pm 3.9\mathrm{e}{-03}$} &  44.2672 \textcolor{mygray}{$\pm 5.2\mathrm{e}{-01}$} \\ \hline 
\end{tabular}
}
\end{table*}
\begin{table*}[h!]
	\centering
	\caption{Average errors for $k$-means}		
	\label{tab:k-means}
	\resizebox{\textwidth}{!}{	
		\renewcommand{\arraystretch}{1.15}
\begin{tabular}{|l|l|l|l|l|l|l|}
\hline
\textbf{Algorithm}       & \textbf{Parameters}   & \multicolumn{5}{c|}{\textbf{Data sets}}            \\ \hline
                         \textbf{\# clusters, $m$, rank $k$} & & \textbf{Friends} & \textbf{Hyper} & \textbf{Logo} & \textbf{Yelp} & \textbf{Social} \\ \hline
(20, 40, 20) & Ours &  0.2934 \textcolor{mygray}{$\pm 5.8\mathrm{e}{-05}$} &  \textbf{0.6110} \textcolor{mygray}{$\pm 1.7\mathrm{e}{-04}$} &  0.1251 \textcolor{mygray}{$\pm 1.2\mathrm{e}{-04}$} &  \textbf{2.4755} \textcolor{mygray}{$\pm 7.2\mathrm{e}{-04}$} &  2.9797 \textcolor{mygray}{$\pm 1.9\mathrm{e}{-03}$} \\ \cline{2-7} 
 & IVY19 &  0.2935 \textcolor{mygray}{$\pm 1.4\mathrm{e}{-04}$} &  0.6117 \textcolor{mygray}{$\pm 2.0\mathrm{e}{-04}$} &  0.1253 \textcolor{mygray}{$\pm 5.8\mathrm{e}{-05}$} &  2.4756 \textcolor{mygray}{$\pm 7.6\mathrm{e}{-04}$} &  2.9810 \textcolor{mygray}{$\pm 9.9\mathrm{e}{-04}$} \\ \cline{2-7} 
 & Classical CS &  0.2933 \textcolor{mygray}{$\pm 1.5\mathrm{e}{-04}$} &  0.6122 \textcolor{mygray}{$\pm 2.6\mathrm{e}{-04}$} &  0.1254 \textcolor{mygray}{$\pm 1.3\mathrm{e}{-04}$} &  2.4759 \textcolor{mygray}{$\pm 7.7\mathrm{e}{-04}$} &  2.9834 \textcolor{mygray}{$\pm 1.8\mathrm{e}{-03}$} \\ \cline{2-7} 
 & Exact SVD &  0.2934 \textcolor{mygray}{$\pm 1.2\mathrm{e}{-04}$} &  0.6113 \textcolor{mygray}{$\pm 2.5\mathrm{e}{-04}$} &  0.1253 \textcolor{mygray}{$\pm 2.5\mathrm{e}{-04}$} &  2.4797 \textcolor{mygray}{$\pm 9.0\mathrm{e}{-04}$} &  \textbf{2.9779} \textcolor{mygray}{$\pm 1.4\mathrm{e}{-03}$} \\ \cline{2-7} 
 & Col. samp. &  \textbf{0.2932} \textcolor{mygray}{$\pm 1.7\mathrm{e}{-04}$} &  0.6129 \textcolor{mygray}{$\pm 3.7\mathrm{e}{-04}$} &  \textbf{0.1250} \textcolor{mygray}{$\pm 2.7\mathrm{e}{-04}$} &  2.4922 \textcolor{mygray}{$\pm 3.8\mathrm{e}{-03}$} &  3.0137 \textcolor{mygray}{$\pm 2.9\mathrm{e}{-03}$} \\ \hline 
(30, 60, 30) & Ours &  0.2604 \textcolor{mygray}{$\pm 7.6\mathrm{e}{-05}$} &  0.5448 \textcolor{mygray}{$\pm 1.4\mathrm{e}{-04}$} &  0.1062 \textcolor{mygray}{$\pm 7.7\mathrm{e}{-05}$} &  \textbf{2.4475} \textcolor{mygray}{$\pm 6.8\mathrm{e}{-04}$} &  2.8765 \textcolor{mygray}{$\pm 1.1\mathrm{e}{-03}$} \\ \cline{2-7} 
 & IVY19 &  0.2605 \textcolor{mygray}{$\pm 4.9\mathrm{e}{-05}$} &  0.5452 \textcolor{mygray}{$\pm 1.6\mathrm{e}{-04}$} &  0.1063 \textcolor{mygray}{$\pm 1.2\mathrm{e}{-04}$} &  2.4476 \textcolor{mygray}{$\pm 7.1\mathrm{e}{-04}$} &  2.8777 \textcolor{mygray}{$\pm 1.2\mathrm{e}{-03}$} \\ \cline{2-7} 
 & Classical CS &  0.2604 \textcolor{mygray}{$\pm 9.3\mathrm{e}{-05}$} &  0.5453 \textcolor{mygray}{$\pm 1.6\mathrm{e}{-04}$} &  0.1062 \textcolor{mygray}{$\pm 4.6\mathrm{e}{-05}$} &  2.4477 \textcolor{mygray}{$\pm 7.4\mathrm{e}{-04}$} &  2.8791 \textcolor{mygray}{$\pm 1.4\mathrm{e}{-03}$} \\ \cline{2-7} 
 & Exact SVD &  0.2605 \textcolor{mygray}{$\pm 1.8\mathrm{e}{-04}$} &  \textbf{0.5446} \textcolor{mygray}{$\pm 2.6\mathrm{e}{-04}$} &  0.1063 \textcolor{mygray}{$\pm 4.4\mathrm{e}{-05}$} &  2.4532 \textcolor{mygray}{$\pm 4.6\mathrm{e}{-04}$} &  \textbf{2.8739} \textcolor{mygray}{$\pm 1.7\mathrm{e}{-03}$} \\ \cline{2-7} 
 & Col. samp. &  \textbf{0.2602} \textcolor{mygray}{$\pm 1.2\mathrm{e}{-04}$} &  0.5459 \textcolor{mygray}{$\pm 3.6\mathrm{e}{-04}$} &  \textbf{0.1059} \textcolor{mygray}{$\pm 1.1\mathrm{e}{-04}$} &  2.4665 \textcolor{mygray}{$\pm 1.1\mathrm{e}{-03}$} &  2.9002 \textcolor{mygray}{$\pm 2.2\mathrm{e}{-03}$} \\ \hline 
\end{tabular}
}
\end{table*}
\paragraph{Wall-Clock Timing.} We remark that online solve time is averaged over the total runtime for solving 10 different problem inputs. Observe that CountSketch-based approaches have the same low online solve time, whereas the solve time for the dense sketch is an order of magnitude larger. 

The offline learning runtime is the time to train a sketch on $\sA_{\train}$ and only applies to learned sketches. Generally, the long \textit{training} times are not problematic because training is only done once and can be completed offline. On the other hand, the online solve time should be as fast as possible.
\begin{table}[H]
	\centering
	\caption{Timing comparison for LRA on Logo with $k = 30, m = 60$}		
	\label{tab:LRA_timing}
	\resizebox{0.75\textwidth}{!}{	
		\renewcommand{\arraystretch}{1}
		\begin{tabular}{|l|c|c|}
					\hline
			\; & \small{\bf Offline learning time (s)} & \small{\bf Online solve time (s)} \\
        \hline
        \small{\bf Ours} & $6300$ (1 h 45 min) & $0.0114$ \\
        \hline 
        \small{\bf IVY9} & $193$ ($3$ min) & $0.0114$ \\
        \hline
        \small{\bf Classical CS} & \xmark & $0.0114 $\\
        \hline
        \small{\bf Exact SVD}& 10 & $0.1750$ \\
        \hline
        \small{\bf Column Sampling}& 10 & $0.0114 $ \\
        \hline
		\end{tabular}}
\end{table}

\paragraph{Conclusion.} We propose the first learning algorithm that also optimizes the locations of the non-zero entries of a CountSketch. We show this algorithm gives better accuracy for low rank approximation than previous work, and apply it other problems such as $k$-means clustering. We show that our algorithm is provably better in the spiked covariance model and for Zipfian matrices. We also show that a sketch monotonicity property holds for these problems, allowing one to combine sketches trained with different models, doing no worse than either model, but potentially doing much better. Our empirical results are strong, highlighting the importance of optimizing not only the values of the non-zero entries but also their positions.  

\section*{Acknowledgments}
The authors would like to thank Piotr Indyk and Honghao Lin for insightful discussions about the project. We also gratefully acknowledge the loan of computational resources from John Dolan, Artur Dubrawski, and Barnabas Poczos. Finally, we extend gratitude to everyone who reviewed our paper, including Changliu Liu, Richard Zhang, and our anonymous reviewers. D. P. Woodruff was supported in part by Office of Naval Research (ONR) grant N00014-18-1-256.

\bibliography{learned_sketch}
\bibliographystyle{abbrvnat}
\appendix
\section{Preliminaries: Theorems and Additional Algorithms}
In this section, we provide several definition and lemmas that are used in proofs of our results for LRA and $k$-means.

\begin{definition}[Affine Embedding]\label{def:affine-embedding}
Given a pair of matrices $A$ and $B$, a matrix $S$ is an {\em affine $\epsilon$-embedding} if for all $X$ of appropriate shape, $\left\| S(AX- B) \right\|^2_F = (1\pm\eps) \left\| AX - B \right\|^2_F$.
\end{definition}

\begin{lemma}[\citep{clarksonwoodruff}; Lemma 40]\label{lem:fast-estimate}
Let $A$ be an $n\times d$ matrix and let $S\in \mathbb{R}^{\O({1/\eps^2}) \times n}$ be a CountSketch. Then with constant probability, $\left\| SA \right\|^2_F = (1\pm \eps) \left\| A \right\|^2_F$.
\end{lemma}
The following result is shown in~\citep{clarksonwoodruff} and sharpened with~\citep{nelson2013osnap,meng2013low}.
\begin{lemma}\label{lem:CS-affine}
Given matrices $A, B$ with $n$ rows, a CountSketch with $\O(\rank(A)^2/\eps^2)$ rows is an affine $\eps$-embedding matrix with constant probability. Moreover, the matrix product $S A$ can be computed in $\O(\nnz(A))$ time, where $\nnz(A)$ denotes the number of non-zero entries of matrix $A$.  
\end{lemma}
Following lemma is from~\citep{sarlos2006improved,clarksonwoodruff}.
\begin{lemma}\label{lem:mult-reg}
Suppose that $A\in \mathbb{R}^{n\times d}$ and $B\in \mathbb{R}^{n\times d'}$. Moreover, let $S\in \mathbb{R}^{m\times n}$ be a CountSketch with $m = \frac{\rank(A)^2}{\eps^2}$. Let $\tilde{X} = \argmin_{\rankk X} \left\| SAX - SB \right\|^2_F$. Then,
\begin{enumerate}[leftmargin=*]
    \item\label{itm:reduction} With constant probability, $\left\| A\tilde{X} - B \right\|^2_F \leq (1+\eps) \min_{\rankk X} \left\| AX - B \right\|^2_F$. In other words, in $\O(\nnz(A) + \nnz(B) + m(d+d'))$ time, we can reduce the problem to a smaller (multi-response regression) problem with $m$ rows whose optimal solution is a $(1+\eps)$-approximate solution to the original instance.
    \item\label{itm:runtim} The $(1+\eps)$-approximate solution $\tilde{X}$ can be computed in time $\O(\nnz(A) + \nnz(B) + mdd' + \min(m^2 d, md^2))$. 
\end{enumerate}
\end{lemma}
Now, we turn our attention to the time-optimal sketching algorithm of LRA. The next lemma is known,
though we include it for completeness \citep{avron2016sharper}: 
\begin{lemma}\label{lem:two-sketch}
Suppose that $S\in \mathbb{R}^{m_S \times n}$ and $R \in \mathbb{R}^{m_R \times d}$ are sparse affine $\eps$-embedding matrices for $(A_k, A)$ and $((SA)^\top, A^\top)$, respectively. Then, 
\begin{align*}
\min_{\rankk X} \left\| AR^\top X SA - A\right\|_F^2 \leq (1+ \eps) \left\| A_k - A\right\|_F^2
\end{align*}
\end{lemma}
\begin{proof}
Consider the following multiple-response regression problem:
\begin{align}\label{eq:simple-regression}
\min_{\rankk X} \left\|A_k X - A\right\|_F^2.
\end{align}
Note that since $X = I_k$ is a feasible solution to  Eq.~\eqref{eq:simple-regression}, $\min_{\rankk X} \left\|A_k X - A\right\|_F^2 = \left\|A_k - A\right\|_F^2$. 
Let $S \in \mathbb{R}^{m_S\times n}$ be a sketching matrix that satisfies the condition of Lemma~\ref{lem:mult-reg} (Item~\ref{itm:reduction}) for $A:= A_k$ and $B := A$. 
By the normal equations, the rank-$k$ minimizer of $\left\|SA_k X - SA\right\|_F^2$ is $(SA_k)^+ SA$.
Hence,
\begin{align}\label{eq:rowsp-SA}
\left\|A_k (SA_k)^+SA - A\right\|_F^2
&\leq  (1+ \eps)\left\|A_k - A\right\|_F^2, 
\end{align}
which in particular shows that a $(1+\eps)$-approximate rank-$k$ approximation of $A$ exists in the row space of $SA$. In other words,
\begin{align}\label{eq:sketch-s}
\min_{\rankk X} \left\|XSA - A\right\|_F^2 \leq (1 + \eps) \left\|A_k - A\right\|_F^2.
\end{align}
Next, let $R\in \mathbb{R}^{m_R\times d}$ be a sketching matrix that satisfies the condition of Lemma~\ref{lem:mult-reg} (Item~\ref{itm:reduction}) for $A := (SA)^\top$ and $B := A^\top$. 
Let $Y$ denote the $\rankk$ minimizer of $\left\|R (SA)^\top X^\top - R A^\top\right\|_F^2$. Hence,
\begin{align}
\left\|(SA)^\top Y^\top - A^\top\right\|_F^2 &\leq (1+\eps) \min_{\rankk X} \left\|XSA - A\right\|_F^2 &&\rhd\text{Lemma~\ref{lem:mult-reg} (Item~\ref{itm:reduction})} \nonumber \\
&\leq (1 + \O(\eps)) \left\|A_k - A\right\|_F^2 &&\rhd\text{Eq.~\eqref{eq:sketch-s}}\label{eq:second-step}
\end{align}
Note that by the normal equations, again $\rowsp(Y^\top) \subseteq \rowsp(R A^\top)$ and we can write $Y = AR^\top Z$ where $\rank(Z)=k$. Thus,
\begin{align*}
\min_{\rankk X} \left\| A R^\top X S A - A\right\|_F^2
\leq \left\| A R^\top Z SA - A\right\|_F^2 &=\left\|(S A)^\top Y^\top - A^\top\right\|_F^2 &&\rhd Y = AR^\top Z \\
&\leq (1 + \O(\eps))\left\|A_k - A\right\|_F^2 &&\rhd \text{Eq.~\eqref{eq:second-step}} 
\end{align*}
\end{proof}

\begin{lemma}[\citep{avron2016sharper}; Lemma~27]\label{lem:main}
For $C \in \mathbb{R}^{p\times m'}, D\in \mathbb{R}^{m\times p'}, G\in \mathbb{R}^{p\times p'}$, the following problem 
\begin{align}
\min_{\rankk Z} \left\| C Z D - G \right\|_F^2 \label{eq:min-Z}
\end{align}
can be solved in $\O(pm'r_C + p'mr_D + pp'(r_D+r_C))$ time, where $r_C = \rank(C) \leq \min\{m', p\}$ and $r_D = \rank(D) \leq \min\{m,p'\}$.
\end{lemma}

\begin{lemma}\label{lem:learned-sketch-runtime}
Let $S\in \mathbb{R}^{m_S\times d}$, $R\in \mathbb{R}^{m_R\times d}$ be CountSketch (CS) matrices such that 
\begin{align}\label{eq:assumption}
    \min_{\rankk X} \left\| A R^\top X SA - A\right\|^2_F \leq (1+\gamma)\left\| A_k - A\right\|_F^2.
\end{align}
Moreover, let $V \in \mathbb{R}^{{m_R^2 \over \beta^2} \times n}$, and $W \in \mathbb{R}^{{m_S^2 \over \beta^2} \times d}$ be CS matrices. Then, Algorithm~\ref{alg:lowrank-sketch} gives a $(1 + \O(\beta + \gamma))$-approximation in time $\nnz(A) + \O(\frac{m_S^4}{\beta^2} + \frac{m_R^4}{\beta^2} + \frac{m_S^2 m_R^2 (m_S + m_R)}{\beta^4} + k(n m_S + d m_R))$ with constant probability. 
\end{lemma}

\begin{proof}
The approximation guarantee follows from Eq.~\eqref{eq:assumption} and the fact that $V$ and $W$ are affine $\beta$-embedding matrices of $AR^\top$ and $SA$, respectively (see Lemma~\ref{lem:CS-affine}). 

The algorithm first computes $C = V AR^\top, D = S A W^\top, G = VAW^\top$ which can be done in time $\O(\nnz(A))$. 
As an example, we bound the time to compute $C = V A R$. Note that since $V$ is a CS, $V A$ can be computed in $\O(\nnz(A))$ time and the number of non-zero entries in the resulting matrix is at most $\nnz(A)$. 
Hence, since $R$ is a CS as well, $C$ can be computed in time $\O(\nnz(A) + \nnz(V A)) = \O(\nnz(A))$.
Then, it takes an extra $\O((m_S^3 + m_R^3 + m_S^2 m_R^2)/\beta^2)$ time to store $C, D$ and $G$ in matrix form. Next, as we showed in Lemma~\ref{lem:main}, the time to compute $Z$ in Algorithm~\ref{alg:lowrank-sketch} is $\O(\frac{m_S^4}{\beta^2} + \frac{m_R^4}{\beta^2} + \frac{m_S^2 m_R^2 (m_S + m_R)}{\beta^4})$. Finally, it takes $\O(\nnz(A) + k(n m_S + d m_R))$ time to compute $Q = AR^\top Z_L$  and $P = Z_R SA$ and to return the solution in the form of $P_{n\times k}Q_{k\times d}$. Hence, the total runtime is 
\begin{align*}
\O(\nnz(A) + \frac{m_S^4}{\beta^2} + \frac{m_R^4}{\beta^2} + \frac{m_S^2 m_R^2 (m_S + m_R)}{\beta^4} + k(n m_S + d m_R))
\end{align*}
\end{proof}

\section{Missing Proofs of Sketch Monotonicity for $k$-Means}\label{sec:mixedsktech-kmeans}
In this section, we prove a number of ``helper'' theorems that are used to prove the main theorem on the sketch monotonicity of $k$-means: Theorem~\ref{thm:mixedsketch-kmeans}.
For the following ``helper'' theorems, we define some notation and apply the useful notion of {\em projection-cost preserving sketch} introduced by~\cite{cohen2015dimensionality}. Again, let $\pi(\cdot)$ be the projection operator, with its subscript specifying the subspace to project onto. Define $\dist^2(A, \mu)$ as the $k$-means loss given cluster centers $\mu = \{\mu_1, \cdots, \mu_k\}$ and, implicitly, a corresponding clustering $\C$: $\dist^2(A, \mu) = \sum_{j \in [n]} \norm{A_j - \mu_i}_2^2$. The clustering $\C$ should be clear from the context. 
Before describing the full proof of sketch monotonicity for $k$-means, we remark that this property automatically provides a worst-case guarantee for the MixedSkecth method.
\paragraph{Remark.} From the sketch monotonicity property, the fact that mixed sketches have worst-case guarantees follows immediately. This is because sketch monotonicity says that adding rows to a given sketch can only improve its performance. A mixed sketch contains a random sketch, so we can see it as a random sketch with rows added. Hence, its performance is at least as good as the performance of the random sketch alone, which is at least as good as its worst-case guarantees.

\begin{definition}[Projection-cost preserving sketch]
$\tilde{A}$ is a \textit{projection-cost preserving sketch} for $A$ if for any low rank projection matrix $P$ and $c$ not dependent on $P$:
$$(1-\eps)\norm{A - PA}_F^2 \leq \norm{\tilde{A} - P \tilde{A}}_F^2 + c \leq (1+\eps)\norm{A - PA}_F^2$$
\end{definition}

\begin{lemma}\label{lem:assumption_implies_approx_lrd}
The fact that $\norm{A - AVV^{\top}}_2^2 \leq \frac{\eta}{k} \norm{A - A_k}_F^2$ for $SA = U\Sigma V^{\top}$ implies $\norm{A - \hat{A}_k}_F^2 \leq (1 + \eta) \norm{A - A_k}_F^2$ for some rank-$k$ matrix $\hat{A}_k \in \row(SA)$.
\begin{proof}
\begin{align}
\norm{A - A_k VV^{\top}}_F^2 &= \norm{A - A_k}_F^2 + \norm{A_k - A_k VV^{\top}}_F^2 + 2\Tr((A - A_k)(A_k - A_k VV^{\top})^{\top}) \label{e1} \\
&= \norm{A - A_k}_F^2 + \norm{A_k ( I- VV^{\top})}_F^2 \label{e2} \\
&\leq  \norm{A - A_k}_F^2 + k \norm{A_k ( I- VV^{\top})}_2^2 \label{e3} \\
&\leq  \norm{A - A_k}_F^2 + k \norm{A ( I- VV^{\top})}_2^2 \label{e4} \\
&\leq  \norm{A - A_k}_F^2 + k\cdot \frac{\eta}{k} \norm{A - A_k}_F^2 \label{e5} \\
&= (1 + \eta) \norm{A - A_k}_F^2 \nonumber
\end{align}
Finally, this implies there is a $(1+\eta)$-approximate $\rankk$ approximation within $\row(SA)$. \\
\eqref{e1} Taking $\norm{A - A_k VV^{\top}}_F^2 = \Tr\left((A - A_k + A_k - A_k VV^{\top})(A - A_k + A_k - A_k VV^{\top})^{\top}\right)$. \\
\eqref{e2} We note that $\col(A - A_k)  \subseteq \col(A_k)^{\bot}$ and $A_k - A_k VV^{\top} = A_k (I - VV^{\top})$ so $\col(A_k - A_k VV^{\top}) \subseteq \col(A_k)$. This implies that $\Tr((A - A_k)(A_k - A_k VV^{\top})^{\top}) = 0$. \\
\eqref{e3} If the rank of $B$ is at most $k$, then $\norm{B}_F^2 \leq k \norm{B}_2^2$. \\
\eqref{e4} 
Using the facts: $\norm{A_k (I - VV^{\top})}_2^2 = \min_{y = (I - VV^{\top})x, \norm{x}_2^2 = 1} \norm{A_k y}_2^2$ and $\norm{A (I - VV^{\top})}_2^2 = \min_{y = (I - VV^{\top})x, \norm{x}_2^2 = 1} \norm{A y}_2^2$. Then, $\forall y, \norm{Ay}_2^2 = \norm{(A_k + A-A_k) y}_2^2 \geq \norm{A_k y}_2^2$ where the last inequality follows from the Pythagorean theorem. Hence, $\norm{A_k (I- VV^\top)}_2^2 \leq \norm{A(I- VV^\top)}_2^2$. \\
\eqref{e5} Using our assumption. 
\end{proof}
\end{lemma}

\begin{corollary}
\label{cor20_equivalent}
Assume we have $A \in \mathbb{R}^{n \times d}$, $j \in \mathbb{Z}^{+}$, $\eta > 0$, and $m = \min(\mathcal{O}(\poly(j/\eta)), d)$. We also have sketch $S \in \mathbb{R}^{m \times n}$ which satisfies $\norm{A - AVV^{\top}}_2^2 \leq \frac{\eta}{j}\norm{A - A_j}_F^2$ for $SA = U \Sigma V^{\top}$. Define $\tilde{A}_m = AVV^{\top}$. Finally, let $X \in \mathbb{R}^{d \times j}$ be a matrix whose columns are orthonormal, and let $Y\in \mathbb{R}^{d\times (d-j)}$ be a matrix with orthonormal columns that spans 
the orthogonal complement of $\col(X)$. Then
\begin{align*}
\norm{A XX^\top - \tilde{A}_m XX^\top}_F^2 \le \eta \cdot \|AY\|_F^2.
\end{align*}
\end{corollary}

\begin{proof}
\begin{align}
\norm{AXX^{\top} - \tilde{A}_m XX^{\top}}_F^2 
&= \norm{A(I - VV^{\top}) XX^{\top}}_F^2 \nonumber \\
&\leq j \norm{A(I - VV^{\top}) XX^{\top}}_2^2 \label{eq1} \\
&\leq j \norm{A(I - VV^{\top})}_2^2 \label{eq2}\\
&\leq j \cdot \mathcal{O}(\frac{\eta}{j}) \norm{A - A_j}_F^2 \label{eq3} \\
&= \mathcal{O}(\eta) \sum_{i = j}^{\min(n, d)} \sigma_i^2 \label{eq3_5} \\
&\leq \mathcal{O}(\eta) \norm{AY}_F^2 \label{eq5} 
\end{align}
\eqref{eq1} Note that $\rank(X) = j$, so $\rank(A(I-VV^{\top})XX^{\top}) \leq j$. We use this fact to bound the Frobenius norm by the operator norm. \\
\eqref{eq2} Using the fact that $XX^{\top}$ is a projection matrix. \\
\eqref{eq3} By the assumption that $\norm{A - AVV^{\top}}_2^2 \leq \frac{\eta}{j}\norm{A - A_j}_F^2$. \\
\eqref{eq3_5} Letting $\sigma_i$ be the singular values of $A$. \\
\eqref{eq5} $\sum_{i = j}^{\min(n, d)} \sigma_i^2  = \min_{Y} \norm{AY}_F^2$ for $Y \in \mathbb{R}^{d \times (d-j)}$ with orthonormal columns.  
\end{proof}

\begin{theorem}\label{thm:approx-dist}
Assume we have $A \in \mathbb{R}^{n \times d}$, $j \in \mathbb{Z}^{+}$, $\eta \in (0, 1]$, and $m = \min(\mathcal{O}(\poly(j/\eta)), d)$. We also have sketch $S \in \mathbb{R}^{m \times n}$ which satisfies $\norm{A - AVV^{\top}}_2^2 \leq \frac{\eta}{k}\norm{A - A_k}_F^2$ for $SA = U \Sigma V^{\top}$. 
Let $\tilde{A}_m = AVV^{\top}$. Then, for any non-empty set $\mu$ contained in a $j$-dimensional subspace, we have:
\begin{align*}
\abs{\dist^2(\tilde{A}_m, \mu) + \norm{\tilde{A}_m - A}_F^2 - \dist^2(A, \mu)} \leq \eta \dist^2(A, \mu)
\end{align*}
\end{theorem}
\begin{proof}
We follow the proof of Theorem 22 in~\citep{feldman2013}, but substitute different analyses in place of Corollaries 16 and 20. The result of~\citeauthor{feldman2013}~involves the {\em best} rank-$m$ approximation of $A$, $A_m$; we will show it for an {\em approximate} rank-$m$ approximation, $\tilde{A}_m$. 

Define $X \in \mathbb{R}^{d \times j}$ with orthonormal columns such that $\col(X) = \Span(\mu)$. Likewise, define $Y \in \mathbb{R}^{d \times (d-j)}$ with orthonormal columns such that $\col(Y) = \Span(\mu)^{\perp}$.
By the Pythagorean theorem:
$$\dist^2(\tilde{A}_m,\mu) = \|\tilde{A}_mY\|_F^2 + \dist^2(\tilde{A}_mXX^T,\mu)$$ and
\begin{equation}\label{eq:acaca}
\dist^2(A,\mu) = \|AY\|_F^2 + \dist^2(AXX^T,\mu).
\end{equation}
Hence,
\begin{align}
    \nonumber& \left|  \left(\dist^2(\tilde{A}_m,\mu) + \norm{A-\tilde{A}_m}_F^2 \right) - \dist^2(A,\mu) \right| \\
 \nonumber=   & \left| \|\tilde{A}_mY\|_F^2  + \dist^2(\tilde{A}_mXX^T,\mu)
+ \left\|A-\tilde{A}_m\right\|_F^2
- \left(\|AY\|_F^2 + \dist^2(AXX^T,\mu) \right) \right|\\
\label{do1}\le & \Big| \|\tilde{A}_mY\|_F^2 + \|A-\tilde{A}_m\|_F^2 - \|AY\|_F^2 \Big| + \left|\dist^2(\tilde{A}_mXX^T,\mu) -\dist^2(AXX^T,\mu)\right| \\
\label{do2}\le & \frac{\varepsilon^2}{8}  \cdot \|AY\|_F^2 + \left| \dist^2(\tilde{A}_mXX^T,\mu) - \dist^2(AXX^T,\mu) \right|\\
\label{do3} \leq& \frac{\varepsilon^2}{8}  \cdot \dist^2(A,\mu) + \left| \dist^2(\tilde{A}_mXX^T,\mu) - \dist^2(AXX^T,\mu) \right| \quad \rhd \text{by Eq.~\eqref{eq:acaca}}\nonumber
\end{align}
\eqref{do1} Triangle inequality. \\
\eqref{do2} Take $\eta$ in Theorem 16 from~\citep{cohen2015dimensionality} as $\eta^2/8$. This theorem implies that $\tilde{A}_m$ is a projection-cost preserving sketch with the $c$ term as $\norm{A - \tilde{A}_m}_F^2$. Specifically, it says $AV$ is a projection-cost preserving sketch, which means $\tilde{A}_m = AVV^{\top}$ is too: $V$ has orthonormal columns so $\norm{(I-P)AV}_F^2 = \norm{(I-P)AVV^{\top}}_F^2$. By Corollary~\ref{cor20_equivalent},
\[
\norm{\tilde{A}_mXX^T-A XX^T}_F^2 \le \frac{\eta^2}{8} \cdot \|AY\|_F^2.
\]
Since $\mu \in \row(X)$, we have $\norm{AY}^2_F \leq \dist^2(A,\mu)$.
Using Corollary 21 from~\citep{feldman2013} while taking $\eta$ as $\eta/4$, $A$ as $\tilde{A}_m XX^T$, and $B$ as $AXX^T$ yields
\begin{equation}
\nonumber | \dist^2(\tilde{A}_mXX^T,\mu)-\dist^2(A XX^T,\mu)| \le \frac{\eta}{4} \cdot \dist^2(A XX^T,\mu) + (1 + \frac{4}{\eta}) \cdot \norm{\tilde{A}_mXX^T-A XX^T}_F^2
\end{equation}
By~\ref{eq:acaca}, $\dist^2(A XX^T,\mu)\leq \dist^2(A,\mu)$.
Finally, we combine the last two inequalities with~\eqref{do3}: 

\begin{align}
\nonumber  &\left|  \left(\dist^2(\tilde{A}_m,\mu) + \norm{A-\tilde{A}_m}_F^2 \right) - \dist^2(A,\mu) \right| \\
	\nonumber \le &  \; \frac{\varepsilon^2}{8}  \cdot \dist^2(A,\mu) + \frac{\eta}{4}\cdot \dist^2(A,\mu) + \frac{\eta^2}{8} \cdot (1+\frac{4}{\eta}) \cdot \dist^2(A,\mu)\\
\nonumber \leq & \; \eta \cdot \dist^2(A,\mu),
\end{align}
where in the last inequality we used the assumption $\eta\leq1$.
\end{proof}

\begin{corollary}\label{cor:diff_centers}
Assume we have $A\in\mathbb{R}^{n\times d}$ and sketch $S \in \mathbb{R}^{m \times n}$ which satisfy $\norm{A - AVV^{\top}}_2^2 \leq \frac{\eta}{k}\norm{A - A_k}_F^2$ for $SA = U \Sigma V^{\top}$. Define $\C=\{C_1,\cdots, C_k\}$ as the optimal $k$-means clustering on $SA$, $\mu_{\C} := \{\mu_{C_1},\cdots, \mu_{C_k}\} \subset \row(SA)$ as the corresponding cluster centers, and $\mu^*_{\C} := \{\mu^*_{C_1}, \cdots, \mu^*_{C_k}\} \subset \mathbbm{R}^{d}$ the optimal cluster centers. For $\eta\in(0,1/3]$, the $k$-means cost of $\C$ with $\mu_{\C}$ as centers is at most $(1+\eta)$ times the $k$-means cost of $\C$ with $\mu_{\C}^*$: 
$$\dist^2(A, \mu_{\C}) \leq (1 + \eta) \dist^2(A, \mu_{\C}^*)$$ 
\end{corollary}

\begin{proof}
By using ${\eta \over 3}$ in Theorem~\ref{thm:approx-dist} with $j$ as $k$, 
\begin{align*}
\abs{\dist^2(\pi_{\row(SA)}(A), \mu_{\C}) + \norm{A - \pi_{\row(SA)}(A)}^2_F - \dist^2(A, \mu_{\C})} \leq \frac{\eta}{3} \dist^2(A, \mu_{\C})
\end{align*}
which implies that 
\begin{align}\label{eq:dist-1}  
(1 - {\eta\over 3}) \dist^2(A, \mu_{\C}) \leq \dist^2(\pi_{\row(SA)}(A), \mu_{\C}) + \norm{A - \pi_{\row(SA)}(A)}^2_F
\end{align}
Likewise, by Theorem~\ref{thm:approx-dist} on $\pi_{\row(SA)}(A)$ and $\mu^*_{\C}$ (and taking $j$ as $k$),
\begin{align*}
\abs{\dist^2(\pi_{\row(SA)}(A), \mu_{\C}^*) + \left\|{A - \pi_{\row(SA)}(A)}\right\|_F^2 - \dist^2(A, \mu_{\C}^*)} \leq \frac{\eta}{3} \dist^2(A, \mu_{\C}^*)
\end{align*}
which implies that
\begin{align}\label{eq:dist-2}
\dist^2(\pi_{\row(SA)}(A), \mu_{\C}^*) + \left\|{A - \pi_{\row(SA)}(A)}\right\|_F^2 \leq (1 + {\eta\over 3}) \dist^2(A, \mu_{\C}^*)
\end{align}
By \eqref{eq:dist-1} and~\eqref{eq:dist-2} together, we have:
\begin{align*}
(1 - {\eta\over 3}) \dist^2(A, \mu_{\C}) 
&\leq \dist^2(\pi_{\row(SA)}(A), \mu_{\C}) + \left\|{A - \pi_{\row(SA)}(A)}\right\|_F^2 \\
&\leq \dist^2(\pi_{\row(SA)}(A), \mu_{\C}^*) + \left\|{A - \pi_{\row(SA)}(A)}\right\|_F^2 \\
&\leq (1 + {\eta\over 3}) \dist^2(A, \mu_{\C}^*)
\end{align*}
Now, $\frac{1 + \eta/3}{1 - \eta/3} \leq 1 + \eta$, so we have $\dist^2(A, \mu_{\C}) \leq (1 + \eta) \dist^2(A, \mu_{\C}^*)$.
\end{proof}

\section{Simple Approximate Comparison Approach}\label{sec:approxcheck}
While the empirical performance of sketching algorithms with our proposed learned CountSketch significantly improves upon the empirical performance of the ones with classical CountSketch (see Section~\ref{sec:evaluation}), the worst-case guarantees that are known for the classical CountSketch \textit{may not} hold for the learned CountSketch. 
To provide a ``best of both worlds'' type of guarantee, we give a very simple ApproxCheck approach which can be applied to LRA and $k$-means, and seems to have been missed in prior work. We remark that ApproxCheck has the same asymptotic runtime as the classical sketching algorithms for our problems of interest.

Here, we run the sketching algorithm twice (once with the learned sketch and once with the classical sketch) and obtain two solutions. Next is we compute the value of both solutions, but only approximately. As the focus of this work is on the design of time-optimal learning-based algorithms, if we computed the exact values, the overall time complexity would increase from the computations on full-sized data. Now, we can perform this approximate evaluation by sketching with the classical CountSketch and leveraging its ``affine embedding'' property (see Definition~\ref{def:affine-embedding}). 

\subsection{ApproxCheck for LRA}\label{sec:approxcheck-lrd}
In this section, we give the pseudocode for the ApproxCheck method and prove that the runtime of this method for LRA is of the same order as the classical time-optimal sketching algorithm of LRA.
\begin{algorithm}[h!]
	\begin{algorithmic}[1]
		\REQUIRE learned sketches $S_L, R_L, V_L, W_L$; classical sketches $S_C, R_C, V_C, W_C$; $\beta$; $A \in \mathbbm{R}^{n \times d}$
        \STATE $P_L, Q_L \gets \SKALG_{\LRA}(S_L, R_L, V_L, W_L, A)$, $P_C Q_C \gets \SKALG_{\LRA}(S_C, R_C, V_C, W_C, A)$
		\STATE Let $S \in \mathbb{R}^{ \O({1/\beta^2}) \times n}, R \in \mathbb{R}^{\O({1/\beta^2}) \times d}$ be classical CountSketch matrices
		\STATE $\Delta_L \leftarrow \left\| S \left( P_L Q_L - A \right) R^{\top} \right\|^2_F$,  $\Delta_C \leftarrow \norm{ S \left( P_C Q_C - A \right)  R^{\top}}^2_F$
		\IF {$\Delta_L \leq \Delta_C$} 
			\STATE {\bf return} $P_L Q_L$
		\ENDIF 
		\STATE {\bf return} $P_C Q_C$
	\end{algorithmic}
	\caption{\textsc{LRA ApproxCheck}}
	\label{alg:learned_sketch_lrd}
\end{algorithm}

\begin{theorem}\label{thm:main-low-rank}
Assume we have data $A \in \mathbb{R}^{n\times d}$, learned sketches  $S_{L} \in \mathbb{R}^{\poly({k \over \eps}) \times n}, R_{L}\in \mathbb{R}^{\poly({k \over \eps}) \times d}, V_{L} \in \mathbb{R}^{\poly({k \over \eps}) \times n}, W_{L}\in \mathbb{R}^{\poly({k \over \eps}) \times d}$ which attain a $(1+ \O(\gamma))$-approximation, classical sketches of the same sizes, $S_C, R_C, V_C, W_C$, which attain a $(1 + \O(\eps))$-approximation, and a trade-off parameter $\beta$. Then, Algorithm~\ref{alg:learned_sketch_lrd} attains a $(1 + \beta + \min(\gamma,\eps))$-approximation in $\mathcal{O}(\nnz(A) + (n+d)\poly({k \over \eps}) + {k^4\over \beta^4}\cdot \poly({k\over \eps}))$ time.
\end{theorem}
\begin{proof}\label{proof:main-low-rank}

Let $(P_L, Q_L)$, $(P_C, Q_C)$ be the approximate rank-$k$ approximations of $A$ in factored form using $(S_L, R_L)$ and $(S_O, R_O)$. Then, clearly, 
\begin{align}\label{eq:low-rank-best-sol}
\min(\left\| P_L Q_L - A \right\|_F^2, \left\| P_C Q_C - A \right\|_F^2) 
= (1 + \mathcal{O}(\min(\eps, \gamma))) \left\| A_k - A\right\|_F^2
\end{align}

Let $\Gamma_L = P_L Q_L - A$, $\Gamma_C = P_C Q_C - A$ and $\Gamma_M = \argmin(\left\| S \Gamma_L R\right\|_F, \left\| S \Gamma_C R\right\|_F)$. Then, 
\begin{align*}
\left\| \Gamma_M \right\|_F^2
&\leq (1+\mathcal{O}(\beta)) \left\| S \Gamma_M R\right\|_F^2 &&\rhd\text{by Lemma~\ref{lem:fast-estimate}}\\
&\leq (1 + \mathcal{O}(\beta)) \cdot \min(\left\| \Gamma_L \right\|^2_F , \left\| \Gamma_C \right\|^2_F) \\
&\leq (1 + \mathcal{O}(\beta+ \min(\eps, \gamma))) \left\| A_k - A\right\|_F^2 &&\rhd\text{by Eq.~\eqref{eq:low-rank-best-sol}}
\end{align*}

\textit{Runtime analysis.} By Lemma~\ref{lem:learned-sketch-runtime}, Algorithm~\ref{alg:lowrank-sketch} computes $P_L, Q_L$ and $P_C, Q_C$ in time
$\mathcal{O}(\nnz(A) +\frac{k^{16}(\beta^2 +\eps^2)}{\eps^{24}\beta^4} + \frac{k^3}{\eps^2} (n + \frac{d k^2}{\eps^4}))$. Next, once we have $P_L, Q_L$ and $P_C, Q_C$, it takes $\mathcal{O}(\nnz(A) + {k\over \beta^4})$ time to compute $\Delta_L$ and $\Delta_C$. 
\begin{align*}
     \mathcal{O}(\nnz(A) +\frac{k^{16}(\beta^2 +\eps^2)}{\eps^{24}\beta^4} + \frac{k^3}{\eps^2} (n + \frac{d k^2}{\eps^4}) + \frac{k}{\beta^4}) 
     =\mathcal{O}(\nnz(A) + (n+d + {k^4 \over \beta^4}) \poly({k\over \eps})).
\end{align*}
\end{proof}

To interpret the above theorem, note that when $\eps \gg k (n+d)^{-4}$, we can set $\beta^{-4} = \mathcal{O}(k (n+d)^{-4})$ so that Algorithm~\ref{alg:learned_sketch_lrd} has the same asymptotic runtime as the best $(1+\eps)$-approximation algorithm of LRA with classical CountSketch. Moreover, Algorithm~\ref{alg:learned_sketch_lrd} is a $(1+o(\eps))$-approximation when the learned sketch outperforms classical sketches, $\gamma = o(\eps)$. On the other hand, when the learned sketches perform poorly, $\gamma = \Omega(\eps)$, the worst-case guarantee of Algorithm~\ref{alg:learned_sketch_lrd} remains $(1+\mathcal{O}(\eps))$.

\subsection{ApproxCheck for $k$-Means}
In this section, we give the pseudocode for the ApproxCheck method and prove that the runtime of this method for $k$-means is of the same order as the classical sketching algorithm of $k$-means.

To understand the following ApproxCheck algorithm for $k$-means (Algorithm~\ref{alg:learned_sketch_kmeans}), note that we can write the $k$-means objective succinctly as $\norm{X^{\top}X A - A}_F^2$, where $X \in \mathbbm{R}^{k \times n}$ is a sparse matrix that represents a clustering $\C$. Let $X_{i, j} = \frac{1}{\sqrt{\abs{\C_i}}}$ if $A_j \in \C_i$ and $0$ everywhere else. Now, $X^{\top}X A$ computes the cluster center corresponding to each sample.  

\begin{algorithm}[h!]
	\begin{algorithmic}[1]
		\REQUIRE learned sketch $S_L$; classical sketch $S_C$; trade-off parameter $\beta$; data $A \in \mathbbm{R}^{n \times d}$
        \STATE $\hat{X}_L \gets \SKALG_{k\text{-means}}(S_L, A), \hat{X}_C \gets \SKALG_{k\text{-means}}(S_C, A)$ 
		\STATE $\Delta_L \leftarrow$ clustering cost of $A$ using the set of centers $\hat{X}_L$ ($\norm{\hat{X}_L^{\top}X A - A}_F^2$)
		\STATE $\Delta_C \leftarrow$ clustering cost of $A$ using the set of centers $\hat{X}_C$ ($\norm{\hat{X}_C^{\top}X A - A}_F^2$)
		\IF {$\Delta_L \leq \Delta_C$} 
			\STATE {\bf return} $\hat{X}_L$
		\ENDIF 
		\STATE {\bf return} $\hat{X}_C$
	\end{algorithmic}
	\caption{\textsc{$k$-means ApproxCheck}}
	\label{alg:learned_sketch_kmeans}
\end{algorithm}

\begin{theorem}\label{thm:main-kmeans}
Assume we have data $A \in \mathbb{R}^{n\times d}$, learned CountSketch $S_{L} \in \mathbb{R}^{\poly({k \over \eps}) \times d}$ which yields a $(1+\gamma)$-approximation, and a classical CountSketch $S_{C} \in \mathbb{R}^{\poly({k \over \eps}) \times d}$ which yields a $(1+ \eps)$-approximation. Further, we assume that $\norm{A - AV_LV_L^{\top}}_2^2 \leq \frac{\eps}{k}\norm{A - A_k}_F^2$ where $AS_L = U_L\Sigma_L V_L^\top$.
We also have an $\alpha$-approximate $k$-means algorithm with time complexity $\mathcal{T}(n, d, k)$. Then, Algorithm~\ref{alg:learned_sketch_kmeans} gives an $\left( \alpha (1 + \min(\gamma,\eps)) \right)$-approximation in time $\mathcal{O}(\mathcal{T}(n, \poly(\frac{k}{\eps}), k) + ndk)$.
\end{theorem}
\begin{proof}
Note that for $S_C$ with $\Omega(\frac{k^2}{\eps^2})$ many rows, with constant probability, $\norm{A - AV_CV_C^{\top}}_2^2 \leq \frac{\eps}{k}\norm{A - A_k}_F^2$ where $AS_C = U_C\Sigma_C V_C^\top$ (see Lemma~18 in~\citep{cohen2015dimensionality}).
By Theorem~\ref{thm:mixedsketch-kmeans}, the solution returned by Algorithm~\ref{alg:kmeans-sketch} using $S_C$ and $S_L$ are respectively an $(\alpha (1+O(\eps)))$-approximate and $(\alpha (1+O(\gamma)))$-approximate $k$-means clustering of $A$. Hence, the solution returned by Algorithm~\ref{alg:learned_sketch_kmeans}, which is the best of $\hat{X}_L$ and $\hat{X}_C$ is an $\alpha(1+\O(\min(\eps, \gamma)))$-approximation.

\textit{Runtime analysis.} Since both $S_C$ and $S_L$ project the input data into $\poly(\frac{k}{\eps})$-dimensional subspaces, it takes $\O(\mathcal{T}(n, \poly(\frac{k}{\eps}),k))$ time to compute solutions $\hat{X}_L$ and $\hat{X}_C$. Next, it takes $\O(ndk)$ time to compute the clustering cost of $A$ using the set of centers $\hat{X}_L$ and $\hat{X}_C$. Hence, the total runtime of Algorithm~\ref{alg:learned_sketch_kmeans} is $\mathcal{O}(\mathcal{T}(n, \poly(\frac{k}{\eps}), k) + ndk)$.
\end{proof}
\section{Location Optimization in CountSketch}\label{sec:greedy-init}
In this section, we analyze the performance of the greedy algorithm for finding the positions of non-zero entries of CountSketch on the two distributions mentioned in Theorem~\ref{thm:greedy_general}: {\em spiked covariance model} and {\em Zipfian}. In particular, we prove that on these two distributions, the CountSkecth whose locations are determined via the greedy approach outperforms the classical CountSkecth. 

\paragraph{Preliminaries and notation.} Left-multiplying $A$ by CountSketch $S \in \mathbb{R}^{m \times n}$ is equivalent to hashing the rows of $A$ to $m$ bins with coefficients in $\{\pm 1\}$. The greedy algorithm proceeds through the rows of $A$ (in some order) and decides which bin to hash to, which we can think of as adding an entry to $S$. We will denote the bins by $b_i$ and their summed contents by $w_i$. 

\subsection{Spiked covariance model with sparse left singular vectors.} \label{sec:greedy_spiked_cov}
To recap, every matrix $A \in \mathbb{R}^{n \times d}$ from the distribution $\sA_{sp}(s,\ell)$ has $s < k$ ``heavy'' rows ($A_{r_1}, \cdots, A_{r_s}$) of norm $\ell >1$. The indices of the heavy rows can be arbitrary, but must be the same for all members of the distribution and are unknown to the algorithm.  The remaining rows (called ``light'' rows) have unit norm. 

In other words: let $\mathcal{R} = \{r_1, \ldots, r_s\}$. For all rows $A_i, i \in [n]$:
\begin{align*}
A_i = 
\left\{
	\begin{array}{ll}
		\ell \cdot v_i & \mbox{if } i \in \mathcal{R}\\
		v_i & \mbox{o.w.}
	\end{array}
\right.
\end{align*}
where $v_i$ is a uniformly random unit vector.

We also assume that $S_r, S_g \in \mathbb{R}^{k \times n}$ and the greedy algorithm proceed in a non-increasing row norm order. 

\paragraph{Proof sketch.} \label{sec:spiked-cov-pf-sketch}
First, we show that the greedy algorithm using a non-increasing row norm ordering will isolate heavy rows (i.e., each is alone in a bin). Then, we conclude by showing that this yields a better $k$-rank approximation error when $d$ is sufficiently large compared to $n$. 

We begin with some preliminary observations that will be of use later. 

It is well-known that a set of uniformly random vectors is {\em $\eps$-almost orthogonal} (i.e., the magnitudes of their pairwise inner products are at most $\eps$).
\begin{observation}\label{obsr:dot-product}
Let $v_1, \cdots, v_n \in \mathbb{R}^d$ be a set of random unit vectors. Then with probability $1-1/\poly(n)$, we have $|\langle v_i, v_j \rangle| \leq 2\sqrt{\log n \over d}, \forall \; i < j\leq n$. 
\end{observation}
We define $\bar{\eps}=2\sqrt{\log n \over d}$.
\begin{observation}\label{obser:norm}
Let $u_{1}, \cdots, u_{t}$ be a set of vectors such that for each pair of $i < j \leq t$, $|\langle \frac{u_{i}}{\norm{u_i}}, \frac{u_{j}}{\norm{u_j}} \rangle| \leq \eps$, and $g_i, \cdots, g_j \in \{-1, 1\}$. Then,  
\begin{align}\label{eq:norm}
\sum_{i=1}^t\left\| u_i \right\|^2_2 - 2\eps\sum_{i<j \leq t} \left\| u_i \right\|_2 \left\| u_j \right\|_2 \leq \left\| \sum_{i=1}^t g_i u_i \right\|_2^2 \leq \sum_{i=1}^t\left\| u_i \right\|^2_2 + 2\eps\sum_{i<j \leq t} \left\| u_i \right\|_2 \left\| u_j \right\|_2
\end{align}
\end{observation}

Next, a straightforward consequence of $\eps$-almost orthogonality is that we can find a QR-factorization of the matrix of such vectors where $R$ (an upper diagonal matrix) has diagonal entries close to $1$ and entries above the diagonal are close to $0$.  

\begin{lemma}\label{lem:almost-orth}
Let $u_1, \cdots, u_t \in \mathbb{R}^d$ be a set of unit vectors such that for any pair of $i<j\leq t$, $|\langle u_i, u_j\rangle| \leq \eps$ where $\eps = O(t^{-2})$. There exists an orthonormal basis $e_1, \cdots, e_t$ for the subspace spanned by $u_1, \cdots, u_t$ such that for each $i\leq t$, $u_i = \sum_{j=1}^i a_{i,j} e_{j}$ where $a_{i,i}^2 \geq 1- \sum_{j=1}^{i-1}j^2 \cdot \eps^2$ and for each $j < i$, $a_{i,j}^2 \leq j^2 \eps^2$.
\end{lemma}
\begin{proof}
We follow the Gram-Schmidt process to construct the orthonormal basis $e_1, \cdots, e_t$ of the space spanned by $u_1, \cdots, u_t$, by first setting $e_1 = u_1$ and then processing $u_2, \cdots, u_t$, one by one. 

The proof is by induction. We show that once the first $j$ vectors $u_1, \cdots, u_j$ are processed the statement of the lemma holds for these vectors. Note that the base case of the induction trivially holds as $u_1 = e_1$. Next, suppose that the induction hypothesis holds for the first $\ell$ vectors $u_1, \cdots, u_\ell$. 
\begin{claim}
For each $j \leq \ell$, $a_{\ell+1, j}^2 \leq j^2 \eps^2$.
\end{claim}
\begin{proof}
The proof of the claim is itself by induction. Note that, for $j=1$ and using the fact that $|\langle u_1, u_{\ell+1}\rangle| \leq \eps$, the statement holds and $a_{\ell+1, 1}^2 \leq \eps^2$. Next, suppose that the statement holds for all $j\leq i<\ell$. Then using that $|\langle u_{i+1}, u_{\ell+1} \rangle| \leq \eps$,
\begin{align*}
    |a_{\ell+1, i+1}| 
    &\leq (|\langle u_{\ell+1}, u_{i+1}| + \sum_{j=1}^{i} |a_{\ell+1, j}| \cdot |a_{i+1, j}|)/|a_{i+1, i+1}| \\
    &\leq (\eps + \sum_{j=1}^{i} j^2\eps^2)/|a_{i+1, i+1}| 
    \quad\rhd\text{by the induction hypothesis on $a_{\ell+1,j}$ for $j\leq i$} \\
    &\leq (\eps + \sum_{j=1}^{i} j^2\eps^2) / ({1- \sum_{j=1}^{i}j^2 \cdot \eps^2})^{1/2} 
    \quad\rhd\text{by induction hypothesis on $a_{i+1,i+1}$}\\
    &\leq (\eps + \sum_{j=1}^{i} j^2\eps^2) \cdot ({1- \sum_{j=1}^{i}j^2 \cdot \eps^2})^{1/2} \cdot (1+ 2 \cdot \sum_{j=1}^{i} j^2\eps^2) \\
    &\leq (\eps + \sum_{j=1}^{i} j^2\eps^2) \cdot (1+ 2 \cdot \sum_{j=1}^{i} j^2\eps^2) \\
    &\leq \eps ((\sum_{j=1}^{i} j^2\eps)\cdot (1+ 4\eps \cdot \sum_{j=1}^{i} j^2\eps)+1) \\
    &\leq \eps (i+1) \quad\rhd\text{by $\eps = O(t^{-2})$}
\end{align*}
\end{proof}

Finally, since $\left\|u_{\ell+1}\right\|^2_2 =1$, $a_{\ell+1, \ell+1}^2 \geq 1- \sum_{j=1}^{\ell} j^2 \eps^2$.
\end{proof}

\begin{corollary}\label{cor:orth-dec}
Suppose that $\bar{\eps} = O(t^{-2})$. There exists an orthonormal basis $e_1,\cdots, e_t$ for the space spanned by the randomly picked vectors $v_1, \cdots, v_t$, of unit norm, so that for each $i$,
$v_i = \sum_{j=1}^i a_{i,j} e_{j}$ where $a_{i,i}^2 \geq 1- \sum_{j=1}^{i-1}j^2 \cdot \bar{\eps}^2$ and for each $j<i$, $a_{i,j}^2 \leq j^2 \cdot \bar{\eps}^2$. 
\end{corollary}
\begin{proof}
The proof follows from Lemma~\ref{lem:almost-orth} and the fact that the set of vectors $v_1, \cdots, v_t$ is $\bar{\eps}$-almost orthogonal (by Observation~\ref{obsr:dot-product}).
\end{proof}

The first main step is to show that the greedy algorithm (with non-increasing row norm ordering) will isolate rows into their own bins until all bins are filled. In particular, this means that the heavy rows (the first to be processed) will all be isolated. 

We note that because we set $\rank(SA) = k$, the $k$-rank approximation cost is the simplified expression $\norm{AVV^{\top} - A}_F^2$, where $U\Sigma V^{\top} = SA$, rather than $\norm{[AV]_{k} V^{\top} - A}_F^2$. This is just the projection cost onto $\row(SA)$. Also, we observe that minimizing this projection cost is the same as maximizing the sum of squared projection coefficients: 
\begin{align*}
\argmin_{S} \norm{A - AVV^{\top}}_F^2 
&= \argmin_{S} \sum_{i \in [n]} \norm{A_i - (\langle A_i, v_1 \rangle v_1 + \ldots + \langle A_i, v_k \rangle v_k)}_2^2
\\ 
&= \argmin_{S} \sum_{i \in [n]} (\norm{A_i}_2^2 -  \sum_{j \in [k]} \langle A_i, v_j \rangle^2) \\
&= \argmax_{S}\sum_{i \in [n]} \sum_{j \in [k]} \langle A_i, v_j \rangle^2
\end{align*}

In the following sections, we will prove that our greedy algorithm makes certain choices by showing that these choices maximize the sum of squared projection coefficients. 

\begin{lemma}\label{lem:heavy-rows}
For any matrix $A$ or batch of matrices $\sA$, at the end of iteration $k$, the learned CountSketch matrix $S$ maps each row to an isolated bin. In particular, heavy rows are mapped to isolated bins.
\end{lemma}
\begin{proof}
For any iteration $i\leq k$, we consider the choice of assigning $A_{i}$ to an empty bin versus an occupied bin. Without loss of generality, let this occupied bin be $b_{i-1}$, which already contains $A_{i-1}$. 

We consider the difference in cost for empty versus occupied. We will do this cost comparison for $A_j$ with $j \leq i - 2$, $j \geq i + 1$, and finally, $j \in \{i-1, i\}$.

First, we let $\{e_1, \ldots, e_i\}$ be an orthonormal basis for $\{A_1, \ldots, A_i\}$ such that for each $r \leq i$, $A_r = \sum_{j=1}^r a_{r,j} e_j$ where $a_{r,r} > 0$. This exists by Lemma~\ref{lem:almost-orth}.
Let $\{e_1, \ldots, e_{i-2}, \bar{e}\}$ be an orthonormal basis for $\{A_1, \ldots, A_{i+2}, A_{i-1} \pm A_i\}$. Now, $\bar{e} = c_0 e_{i-1} + c_1 e_i$ for some $c_0, c_1$ because $(A_{i-1} \pm A_i) - \proj_{\{e_1, \ldots, e_{i-2}\}}(A_{i-1} \pm A_i) \in \Span(e_{i-1}, e_i)$. We note that $c_0^2 + c_1^2 = 1$ because we let $\bar{e}$ be a unit vector. We can find $c_0, c_1$ to be:
\begin{align*}
c_0 = {a_{i-1, i-1} + a_{i, i-1} \over \sqrt{(a_{i-1, i-1} + a_{i, i-1})^2 + a_{i,i}^2}},\quad
c_1 = {a_{i,i} \over \sqrt{(a_{i-1, i-1} + a_{i, i-1})^2 + a_{i,i}^2}}
\end{align*}

\begin{enumerate}[leftmargin=*]
    \item $j \leq i - 2$: The cost is zero for both cases because $A_j \in \Span(\{e_1,\ldots,e_{i-2}\})$.  
    \item $j \geq i + 1$: We compare the rewards (sum of squared projection coefficients) and find that $\{e_1, \ldots, e_{i-2}, \bar{e}\}$ is no better than $\{e_1, \ldots, e_i\}$.
    \begin{align*}
    \langle A_j, \bar{e} \rangle^2 
    &= (c_0 \langle A_j, e_{i-1} \rangle + c_1 \langle A_j, e_{i} \rangle)^2\\ 
    &\leq (c_1^2 + c_0^2) (\langle A_j, e_{i-1} \rangle^2 + \langle A_j, e_{i} \rangle^2) &&\rhd\text{Cauchy-Schwarz inequality}\\ 
    &= \langle A_j, e_{i-1} \rangle^2 + \langle A_j, e_{i} \rangle^2
    \end{align*}
    \item $j \in \{i-1, i\}$: 
    We compute the sum of squared projection coefficients of $A_{i-1}$ and $A_i$ onto $\bar{e}$: 
    \begin{align}
    &\;({1 \over (a_{i-1, i-1} + a_{i, i-1})^2 + a_{i,i}^2})\cdot (a_{i-1, i-1}^2 (a_{i-1, i-1} + a_{i, i-1})^2 + (a_{i,i-1}(a_{i-1, i-1} + a_{i, i-1}) + a_{i,i} a_{i,i})^2) \nonumber \\
    =&({1 \over (a_{i-1, i-1} + a_{i, i-1})^2 + a_{i,i}^2})\cdot ((a_{i-1, i-1} + a_{i, i-1})^2 (a_{i-1,i-1}^2 + a_{i, i-1}^2) \nonumber\\
    &\quad + a_{i,i}^4 + 2a_{i,i-1} a_{i,i}^2 (a_{i-1, i-1} + a_{i, i-1}))\label{eq:sum}
    \end{align}
    On the other hand, the sum of squared projection coefficients of $A_{i-1}$ and $A_i$ onto $e_{i-1} \cup e_{i}$ is:
    \begin{align}
    &({(a_{i-1, i-1} + a_{i, i-1})^2 + a_{i,i}^2 \over (a_{i-1, i-1} + a_{i, i-1})^2 + a_{i,i}^2})\cdot (a_{i-1, i-1}^2 + a_{i, i-1}^2 + a_{i,i}^2)\label{eq:both} 
    \end{align}
    Hence, the difference between the sum of squared projections of $A_{i-1}$ and $A_i$ onto $\bar{e}$ and $e_{i-1}\cup e_i$ is (\eqref{eq:both} - \eqref{eq:sum})
    \begin{align*}
    &\quad {a_{i,i}^2 ((a_{i-1, i-1} + a_{i, i-1})^2 + a_{i-1, i-1}^2 + a_{i, i-1}^2 - 2a_{i,i-1} (a_{i-1, i-1} + a_{i, i-1})) \over (a_{i-1, i-1} + a_{i, i-1})^2 + a_{i,i}^2} \\
    &= {2 a_{i,i}^2 a_{i-1,i-1}^2 \over (a_{i-1, i-1} + a_{i, i-1})^2 + a_{i,i}^2} > 0
    \end{align*}
\end{enumerate}
Thus, we find that $\{e_1, \ldots, e_i\}$ is a strictly better basis than $\{e_1, \ldots, e_{i-2}, \bar{e}\}$. This means the greedy algorithm will choose to place $A_i$ in an empty bin. 
\end{proof}

Next, we show that none of the rows left to be processed (all light rows) will be assigned to the same bin as a heavy row. The main proof idea is to compare the cost of ``colliding'' with a heavy row to the cost of ``avoiding'' the heavy rows. Specifically, we compare the \textit{decrease} (before and after bin assignment of a light row) in sum of squared projection coefficients, lower-bounding it in the former case and upper-bounding it in the latter.   

We introduce some results that will be used in Lemma~\ref{lem:light-rows}. 

\begin{claim}\label{clm:not-proc}
Let $A_{k+r}, r \in [1, \ldots, n-k]$ be a light row not yet processed by the greedy algorithm.
Let $\{e_1,\ldots,e_k\}$ be the Gram-Schmidt basis for the current $\{w_1,\ldots,w_k\}$. Let $\beta = \O(n^{-1}k^{-3})$ upper bound the inner products of normalized $A_{k+r}, w_1, \ldots, w_k$. Then, for any bin $i$, $\langle e_i, A_{k+r} \rangle^2 \leq \beta^2 \cdot k^2$.
\end{claim}
\begin{proof}
This is a straightforward application of Lemma~\ref{lem:almost-orth}.
From that, we have $\langle A_{k + r}, e_i \rangle^2 \leq i^2 \beta^2$, for $i \in [1,\ldots,k]$, which means $\langle A_{k + r}, e_i \rangle^2 \leq k^2 \beta^2$. 
\end{proof}

\begin{claim}\label{clm:proc-not-cont}
Let $A_{k+r}$ be a light row that has been processed by the greedy algorithm. 
Let $\{e_1,\ldots,e_k\}$ be the Gram-Schmidt basis for the current $\{w_1,\ldots,w_k\}$. If $A_{k+r}$ is assigned to bin $b_{k-1}$ (w.l.o.g.), the squared projection coefficient of $A_{k + r}$ onto $e_i, i \neq k - 1$ is at most $4\beta^2 \cdot k^2$, where $\beta = \O(n^{-1}k^{-3})$ upper bounds the inner products of normalized $A_{k+r}, w_1, \cdots, w_k$.
\end{claim}
\begin{proof} 
Without loss of generality, it suffices to bound the squared projection of $A_{k+r}$ onto the direction of $w_k$ that is orthogonal to the subspace spanned by $w_1, \cdots, w_{k-1}$. Let $e_1, \cdots, e_k$ be an orthonormal basis of $w_1, \cdots, w_k$ guaranteed by Lemma~\ref{lem:almost-orth}. Next, we expand the orthonormal basis to include $e_{k+1}$ so that we can write the normalized vector of $A_{k+r}$ as $v_{k+r}=\sum_{j=1}^{k+1} b_j e_j$. By a similar approach to the proof of Lemma~\ref{lem:almost-orth}, for each $j\leq k-2$, $b_j\leq \beta^2 j^2$. Next, since $|\langle w_k , v_{k+r}\rangle| \leq \beta$,
\begin{align*}
|b_k| &\leq {1 \over |\langle w_k, e_k\rangle|} \cdot (|\langle w_k , v_{k+r}\rangle| + \sum_{j=1}^{k-1} |b_j \cdot \langle w_k, e_j\rangle|)\\
&\leq {1 \over \sqrt{1 - \sum_{j=1}^{k-1} \beta^2 \cdot j^2}} \cdot (\beta + \sum_{j=1}^{k-2} \beta^2 \cdot j^2 + (k-1)\cdot \beta) &&\rhd |b_{k-1}|\leq 1\\
&= {\beta + \sum_{j=1}^{k-2} \beta^2 \cdot j^2 \over \sqrt{1 - \sum_{j=1}^{k-1} \beta^2 \cdot j^2}} + (k-1) \beta \\
&\leq 2(k-1) \beta - {\beta^2 (k-1)^2 \over \sqrt{1 - \sum_{j=1}^{k-1} \beta^2 \cdot j^2}} &&\rhd \text{similar to the proof of Lemma~\ref{lem:almost-orth}}\\
&<2\beta \cdot k    
\end{align*}
Hence, the squared projection of $A_{k+r}$ onto $e_k$ is at most $4\beta^2 \cdot k^2 \cdot \left\| A_{k+r} \right\|^2_2$. We assumed $\norm{A_{k+r}} =1$;  hence, the squared projection of $A_{k+r}$ onto $e_k$ is at most $4\beta^2 \cdot k^2$. 
\end{proof}

\begin{claim}\label{clm:proc-cont}
We assume that the absolute values of the inner products of vectors in $v_1, \cdots, v_n$ are at most $\bar{\eps} < 1/ (n^2 \sum_{A_i\in b} \left\| A_i \right\|_2)$ and the absolute values of the inner products of the normalized vectors of $w_1, \cdots, w_k$ are at most $\beta = \O(n^{-3} k^{-{3 \over 2}})$.
Suppose that bin $b$ contains the row $A_{k+r}$. Then, the squared projection of $A_{k + r}$ onto the direction of $w$ orthogonal to $\Span(\{w_1,\cdots, w_k\} \setminus \{w\})$ is at most ${\left\| A_{k+r} \right\|^4_2 \over \left\| w\right\|^2_2} + \O(n^{-2})$ and is at least ${\left\| A_{k+r} \right\|^4_2 \over \left\| w\right\|^2_2} - \O(n^{-2})$.
\end{claim}
\begin{proof}
Without loss of generality, we assume that $A_{k+r}$ is mapped to $b_k$; $w = w_k$. First, we provide an upper and a lower bound for $|\langle v_{k+r}, \bar{w}_k\rangle|$ where for each $i\leq k$, we let $\bar{w}_i = {w_i \over \left\| w_i \right\|_2}$ denote the normalized vector of $w_i$. Recall that by definition $v_{k+r} = {A_{k+r} \over \left\| A_{k+r} \right\|_2}$.  
\begin{align}
|\langle \bar{w}_k, v_{k+r}\rangle| 
&\leq {\left\| A_{k+r} \right\|_2 + \sum_{A_i\in b_k} \bar{\eps} \left\| A_i \right\|_2 \over \left\| w_k\right\|_2} \nonumber \\ 
&\leq  {\left\| A_{k+r} \right\|_2 + n^{-2} \over \left\| w_k\right\|_2} &&\rhd \text{by $\bar{\eps} < {n^{-2} \over \sum_{A_i\in b_k} \left\| A_i \right\|_2}$} \nonumber \\  
&\leq {\left\| A_{k+r} \right\|_2 \over \left\| w_k\right\|_2} + n^{-2} &&\rhd \text{$\left\| w_k \right\|_2 \geq 1$} \label{eq:up} \\
\nonumber \\ 
|\langle \bar{w}_k, v_{k+r}\rangle| 
&\geq {\left\| A_{k+r} \right\|_2 - \sum_{A_i\in b_k} \left\| A_i \right\|_2\cdot \bar{\eps} \over \left\| w_k\right\|_2} \nonumber \\
&\geq {\left\| A_{k+r} \right\|_2 \over \left\| w_k\right\|_2} - n^{-2} \label{eq:low}
\end{align}
Now, let $\{e_1, \cdots, e_{k}\}$ be an orthonormal basis for the subspace spanned by $\{w_1, \cdots, w_k\}$ guaranteed by Lemma~\ref{lem:almost-orth}. 
Next, we expand the orthonormal basis to include $e_{k+1}$ so that we can write $v_{k+r}=\sum_{j=1}^{k+1} b_j e_j$. 
By a similar approach to the proof of Lemma~\ref{lem:almost-orth}, we can show that for each $j\leq k-1$, $b_j^2\leq \beta^2 j^2$. Moreover,
\begin{align*}
|b_k| &\leq {1 \over |\langle \bar{w}_k, e_k\rangle|} \cdot (|\langle \bar{w}_k , v_{k+r}\rangle| + \sum_{j=1}^{k-1} |b_j \cdot \langle \bar{w}_k, e_j\rangle|)\\
&\leq {1 \over \sqrt{1 - \sum_{j=1}^{k-1} \beta^2 \cdot j^2}} \cdot (|\langle \bar{w}_k , v_{k+r}\rangle| + \sum_{j=1}^{k-1} \beta^2 \cdot j^2) &&\rhd\text{by Lemma~\ref{lem:almost-orth}}\\
&\leq {1 \over \sqrt{1 - \sum_{j=1}^{k-1} \beta^2 \cdot j^2}} \cdot (n^{-2} + {\left\| A_{k+r} \right\|_2 \over \left\| w_k\right\|_2} + \sum_{j=1}^{k-1} \beta^2 \cdot j^2) &&\rhd\text{by~\eqref{eq:up}} \\
&<\beta \cdot k + {1 \over \sqrt{1 - \beta^2 k^3}} \cdot (n^{-2} + {\left\| A_{k+r} \right\|_2 \over \left\| w_k\right\|_2}) &&\rhd\text{similar to the proof of Lemma~\ref{lem:almost-orth}}\\  
&\leq \O(n^{-2}) + (1 + \O(n^{-2})){\left\| A_{k+r} \right\|_2 \over \left\| w_k\right\|_2} &&\rhd\text{by $\beta = \O(n^{-3}k^{-{3\over 2}})$}\\
&\leq {\left\| A_{k+r} \right\|_2 \over \left\| w_k\right\|_2} + \O(n^{-2}) &&\rhd {\left\| A_{k+r} \right\|_2 \over \left\| w_k\right\|_2} \leq 1
\end{align*}
and,
\begin{align*}
|b_k| &\geq {1 \over |\langle \bar{w}_k, e_k\rangle|} \cdot (|\langle \bar{w}_k , v_{k+r}\rangle| - \sum_{j=1}^{k-1} |b_j \cdot \langle \bar{w}_k, e_j\rangle|)\\ 
&\geq |\langle \bar{w}_k , v_{k+r}\rangle| - \sum_{j=1}^{k-1} \beta^2 \cdot j^2 &&\rhd\text{since $|\langle \bar{w}_k, e_k\rangle|\leq 1$}\\
&\geq {\left\| A_{k+r} \right\|_2 \over \left\| w_k\right\|_2} - n^{-2} - \sum_{j=1}^{k-1} \beta^2 \cdot j^2 &&\rhd\text{by~\eqref{eq:low}} \\
&\geq {\left\| A_{k+r} \right\|_2 \over \left\| w_k\right\|_2} - \O(n^{-2})  &&\rhd\text{by $\beta = \O(n^{-3} k^{-{3\over 2}})$}  
\end{align*}
Hence, the squared projection of $A_{k+r}$ onto $e_k$ is at most ${\left\| A_{k+r} \right\|^4_2 \over \left\| w_k\right\|^2_2} + \O(n^{-2})$ and is at least ${\left\| A_{k+r} \right\|^4_2 \over \left\| w_k\right\|^2_2} - \O(n^{-2})$.
\end{proof}

Now, we show that at the end of the algorithm no light row will be assigned to the bins that contain heavy rows.

\begin{lemma}\label{lem:light-rows}
We assume that the absolute values of the inner products of vectors in $v_1, \cdots, v_n$ are at most $\bar{\eps} < \min\{n^{-2} k^{-{5\over 3}}, (n\sum_{A_i\in w} \left\| A_i \right\|_2)^{-1}\}$. At each iteration $k+r$, the greedy algorithm will assign the light row $A_{k+r}$ to a bin that does not contain a heavy row.
\end{lemma}
\begin{proof}
The proof is by induction. Lemma~\ref{lem:heavy-rows} implies that no light row has been mapped to a bin that contains a heavy row for the first $k$ iterations. Next, we assume that this holds for the first $k+r-1$ iterations and show that is also must hold for the $(k+r)$-th iteration. 

To this end, we compare the sum of squared projection coefficients when $A_{k+r}$ avoids and collides with a heavy row.
 
First, we upper bound $\beta = \max_{i\neq j\leq k} |\langle w_i, w_j \rangle| / (\left\| w_i \right\|_2 \left\| w_j \right\|_2)$. Let $c_i$ and $c_j$ respectively denote the number of rows assigned to $b_i$ and $b_j$.
\begin{align*}
\beta = \max_{i\neq j\leq k} {|\langle w_i, w_j \rangle| \over \left\| w_i \right\|_2 \left\| w_j \right\|_2} 
&\leq {c_i \cdot c_j \cdot \bar{\eps} \over \sqrt{c_i - 2 \bar{\eps}c_i^2} \cdot  \sqrt{c_j - 2 \bar{\eps}c_j^2}} &\rhd~\text{Observation \ref{obser:norm}}\\
&\leq {16 \bar{\eps}\sqrt{c_i c_j}} &\rhd \bar{\eps} \leq n^{-2} k^{-5/3} \\
&\leq n^{-1} k^{-{5 \over 3}} &\rhd \bar{\eps} \leq n^{-2} k^{-5/3}
\end{align*}

\paragraph{1. If $A_{k+r}$ is assigned to a bin that contains $c$ light rows and no heavy rows.}
In this case, the projection loss of the heavy rows $A_1, \cdots, A_s$ onto $\row(SA)$ remains zero. Thus, we only need to bound the change in the sum of squared projection coefficients of the light rows before and after iteration $k+r$.
Without loss of generality, let $w_k$ denote the bin that contains $A_{k+r}$. Since $\sS_{k-1} = \Span(\{w_1, \cdots, w_{k-1}\})$ has not changed, we only need to bound the difference in cost between projecting onto the component of $w_k - A_{k+r}$ orthogonal to $\sS_{k-1}$ and the component of $w_k$ orthogonal to $\sS_{k-1}$, respectively denoted as $e_k$ and $\bar{e}_k$.  
\begin{enumerate}[leftmargin=*]
\item\label{case:not-proc} By Claim~\ref{clm:not-proc}, for the light rows that are not yet processed (i.e., $A_j$ for $j > k+r$), the squared projection of each onto $e_k$ is at most $\beta^2 k^2$. Hence, the total decrease in the squared projection is at most $(n-k-r)\cdot \beta^2 k^2$.
\item\label{case:proc-not-cont} By Claim~\ref{clm:proc-not-cont}, for the processed light rows that are not mapped to the last bin, the squared projection of each onto $e_k$ is at most $4\beta^2 k^2$. Hence, the total decrease in the squared projection cost is at most $(r-1)\cdot 4 \beta^2 k^2$.
\item\label{case:proc-cont} For each row $A_{i} \neq A_{k+r}$ that is mapped to the last bin, by Claim~\ref{clm:proc-cont} and the fact $\norm{A_i}^4_2 = \norm{A_i}^2_2 = 1$, the squared projection of $A_{i}$ onto $e_k$ is at most ${\left\| A_{i} \right\|^2_2 \over \left\| w_k - A_{k+r} \right\|^2_2} + \O(n^{-2})$ and the squared projection of $A_{i}$ onto $\bar{e}_k$ is at least ${\left\| A_{i} \right\|^2_2 \over \left\| w_k \right\|^2_2} - \O(n^{-2})$.

Moreover, the squared projection of $A_{k+r}$ onto $e_k$ compared to $\bar{e}_k$ increases by at least $({\left\| A_{k+r} \right\|^2_2 \over \left\| w_k\right\|^2_2} - \O(n^{-2})) - \O(n^{-2}) = {\left\| A_{k+r} \right\|^2_2 \over \left\| w_k\right\|^2_2} - \O(n^{-2})$.

Hence, the total squared projection of the rows in the bin $b_k$ decreases by at least:
\begin{align*}
& (\sum_{A_i \in w_k/\{A_{r+k}\}} {\left\| A_i \right\|^2_2 \over \left\| w_k - A_{r+k} \right\|^2_2} + \O(n^{-2})) - (\sum_{A_i \in w_k} {\left\| A_i \right\|^2_2 \over \left\| w_k \right\|^2_2} - \O(n^{-2}))\\
\leq & {\left\| w_k - A_{r+k} \right\|^2_2 + \O(n^{-1}) \over \left\| w_k - A_{r+k} \right\|^2_2} - {\left\| w_k \right\|^2_2 - \O(n^{-1}) \over \left\| w_k \right\|^2_2}  + \O(n^{-1}) &\rhd\text{ by Observation \ref{obser:norm}}\\
\leq &  \O(n^{-1})
\end{align*}  
\end{enumerate}
Hence, summing up the bounds in items~\ref{case:not-proc} to~\ref{case:proc-cont} above, the total decrease in the sum of squared projection coefficients is at most $\O(n^{-1})$.

\paragraph{2. If $A_{k+r}$ is assigned to a bin that contains a heavy row.}  
Without loss of generality, we can assume that $A_{k+r}$ is mapped to $b_k$ that contains the heavy row $A_s$. In this case, the distance of heavy rows $A_1, \cdots, A_{s-1}$ onto the space spanned by the rows of $SA$ is zero. Next, we bound the amount of change in the squared distance of $A_{s}$ and light rows onto the space spanned by the rows of $SA$.

Note that the $(k-1)$-dimensional space corresponding to $w_1, \cdots, w_{k-1}$ has not changed. Hence, we only need to bound the decrease in the projection distance of $A_{k+r}$ onto $\bar{e}_k$ compared to $e_k$ (where $\bar{e}_k, e_k$ are defined similarly as in the last part).  

\begin{enumerate}[leftmargin=*]
\item For the light rows other than $A_{k+r}$, the squared projection of each onto $e_k$ is at most $\beta^2 k^2$. Hence, the total increase in the squared projection of light rows onto $e_k$ is at most $(n-k)\cdot \beta^2 k^2 = \O(n^{-1})$.
\item By Claim~\ref{clm:proc-cont}, the sum of squared projections of $A_s$ and $A_{k+r}$ onto $e_k$ decreases by at least 
\begin{align*}
&\left\|A_s\right\|^2_2 - ({\left\|A_s\right\|^4_2 + \left\|A_{k+r}\right\|^4_2 \over \left\|A_s + A_{r+k}\right\|^2_2} + \O(n^{-1})) \\
\geq &\left\|A_s\right\|^2_2 - ({\left\|A_s\right\|^4_2 + \left\|A_{k+r}\right\|^4_2 \over \left\|A_s\right\|_2^2 + \left\|A_{r+k}\right\|^2_2 -n^{-\O(1)}} + \O(n^{-1})) &&\rhd\text{by Observation~\ref{obser:norm}}\\
\geq &{\left\|A_{r+k}\right\|^2_2 (\left\|A_s\right\|^2_2 - \left\|A_{k+r}\right\|^2_2) - \left\| A_s \right\|^2_2 \cdot \O(n^{-1}) \over \left\|A_s\right\|_2^2 + \left\|A_{r+k}\right\|^2_2 - \O(n^{-1})} - \O(n^{-1})\\
\geq &{\left\|A_{r+k}\right\|^2_2 (\left\|A_s\right\|^2_2 - \left\|A_{k+r}\right\|^2_2) - \left\| A_s \right\|^2_2 \cdot \O(n^{-1}) \over \left\|A_s\right\|_2^2 + \left\|A_{r+k}\right\|^2_2} - \O(n^{-1}) \\
\geq &{\left\|A_{r+k}\right\|^2_2 (\left\|A_s\right\|^2_2 - \left\|A_{k+r}\right\|^2_2) \over \left\|A_s\right\|_2^2 + \left\|A_{r+k}\right\|^2_2} - \O(n^{-1}) \\
\geq &{\left\|A_{r+k}\right\|^2_2 (1 - (\left\|A_{k+r}\right\|^2_2/\left\|A_s\right\|^2_2)) \over 1 + (\left\|A_{r+k}\right\|^2_2/\left\|A_s\right\|^2_2)} - \O(n^{-1}) \\
\geq & \left\|A_{r+k}\right\|^2_2 (1 - {\left\| A_{k+r} \right\|_2 \over \left\|A_s\right\|_2}) - \O(n^{-1}) &&\rhd {1 - \eps^2 \over 1+\eps^2} \geq 1-\eps
\end{align*} 
\end{enumerate}
Hence, in this case, the total decrease in the squared projection is at least 
\begin{align*}
\left\|A_{r+k}\right\|^2_2 (1 - {\left\| A_{k+r} \right\|_2 \over \left\|A_s\right\|_2}) - \O(n^{-1})  
&= 1 - {\left\| A_{k+r} \right\|_2 \over \left\|A_s\right\|_2}) - \O(n^{-1}) &&\rhd \left\|A_{r+k}\right\|_2 = 1 \\
&= 1 -(1/\sqrt{\ell}) - \O(n^{-1}) &&\rhd \left\|A_{s}\right\|_2 = \sqrt{\ell} 
\end{align*} 
Thus, for a sufficiently large value of $\ell$, the greedy algorithm will assign $A_{k+r}$ to a bin that only contains light rows. This completes the inductive proof and in particular implies that at the end of the algorithm, heavy rows are assigned to isolated bins.
\end{proof}

\begin{corollary}\label{cor:greedy-cost}
The approximation loss of the best $\rankk$ approximate solution in the rowspace $S_gA$ for $A\sim \sA_{sp}(s,\ell)$ where $A\in \mathbb{R}^{n\times d}$ for $d = \Omega(n^4 k^4 \log n)$ and $S_g$ is the CountSketch constructed by the greedy algorithm with non-increasing order is at most $n-s$.
\end{corollary}
\begin{proof}
First, we need to show that the absolute values of the inner products of vectors in $v_1, \cdots, v_n$ are at most $\bar{\eps} < \min\{n^{-2} k^{-2}, (n\sum_{A_i\in w} \left\| A_i \right\|_2)^{-1}\}$ so that we can apply Lemma~\ref{lem:light-rows}.
To show this, note that by Observation~\ref{obsr:dot-product}, $\bar{\eps} \leq 2\sqrt{\log n \over d} \leq n^{-2} k^{-2}$ since $d = \Omega(n^4 k^4 \log n)$.
The proof follows from Lemma~\ref{lem:heavy-rows} and Lemma~\ref{lem:light-rows}. Since all heavy rows are mapped to isolated bins, the projection loss of the light rows is at most $n-s$. 
\end{proof}

Next, we bound the Frobenius norm error of the best $\rankk$-approximation solution constructed by the standard CountSketch with a randomly chosen sparsity pattern.
\begin{lemma}\label{lem:random-pattern}
Let $s=\alpha k$ where $0.7<\alpha < 1$. The expected squared loss of the best $\rankk$ approximate solution in the rowspace $S_r A$ for $A\in \mathbb{R}^{n\times d}$ $\sim \sA_{sp}(s,\ell)$ where $d=\Omega(n^6 \ell^2)$ and $S_r$ is the sparsity pattern of CountSketch is chosen uniformly at random is at least $n + {\ell k \over 4e} - (1+\alpha) k - n^{-\O(1)}$.
\end{lemma} 
\begin{proof}
We can interpret the randomized construction of the CountSketch as a ``balls and bins'' experiment. In particular, considering the heavy rows, we compute the expected number of bins (i.e., rows in $S_r A$) that contain a heavy row. Note that the expected number of rows in $S_rA$ that do not contain any heavy row is $k\cdot (1 - {1\over k})^s \geq k\cdot e^{-{s \over k-1}}$. Hence, the number of rows in $S_r A$ that contain a heavy row of $A$ is at most $k(1 - e^{-{s \over k-1}})$. Thus, at least $s - k(1 - e^{-{s \over k-1}})$ heavy rows are not mapped to an isolated bin (i.e., they collide with some other heavy rows). Then, it is straightforward to show that the squared loss of each such row is at least $\ell-n^{-\O(1)}$. 
\begin{claim}
Suppose that heavy rows $A_{r_1} ,\cdots, A_{r_c}$ are mapped to the same bin via a CountSketch $S$. Then, the total squared distances of these rows from the subspace spanned by $SA$ is at least $(c-1)\ell - \O(n^{-1})$.
\end{claim}
\begin{proof}
Let $b$ denote the bin that contains the rows $A_{r_1}, \cdots, A_{r_c}$ and suppose that it has $c'$ light rows as well. Note that by Claim~\ref{clm:proc-not-cont} and Claim~\ref{clm:proc-cont}, the squared projection of each row $A_{r_i}$ onto the subspace spanned by the $k$ bins is at most 
\begin{align*}
&{\left\| A_{h_i} \right\|^4_2 \over \left\| w \right\|^2_2} + \O(n^{-1}) \\ 
\leq &{\ell^2 \over c \ell + c' - 2\bar{\eps} (c^2\ell + cc'\sqrt{\ell} + c'^2)} + \O(n^{-1}) \\
\leq &{\ell^2 \over c\ell -n^{-\O(1)}} + n^{-\O(1)} &&\rhd\text{by $\bar{\eps} \leq n^{-3}\ell^{-1}$} \\
\leq &{\ell^2 \over c^2 \ell^2} \cdot (c\ell + \O(n^{-1}) + \O(n^{-1}) \\
\leq &{\ell \over c} + \O(n^{-1})
\end{align*} 
Hence, the total squared loss of these $c$ heavy rows is at least $c\ell - c \cdot ({\ell \over c} + \O(n^{-1})) \geq (c-1)\ell - \O(n^{-1})$.
\end{proof}
Thus, the expected total squared loss of heavy rows is at least:
\begin{align*}
&\ell \cdot (s - k (1 - e^{- {s\over k-1}})) - s \cdot n^{-\O(1)} \\
\geq &\ell \cdot k(\alpha - 1 + e^{-\alpha}) - \ell \alpha - n^{-\O(1)}  &&\rhd s = \alpha\cdot (k-1) \text{ where $0.7<\alpha<1$} \\
\geq &{\ell k \over 2e} - \ell  - n^{-\O(1)} &&\rhd \alpha\geq 0.7 \\
\geq &{\ell k\over 4e} - \O(n^{-1}) &&\rhd \text{assuming $k>4e$} 
\end{align*} 

Next, we compute a lower bound on the expected squared loss of the light rows. Note that Claim~\ref{clm:proc-not-cont} and Claim~\ref{clm:proc-cont} imply that when a light row collides with other rows, its contribution to the total squared loss (note that the loss accounts for the amount it decreases from the squared projection of the other rows in the bin as well) is at least $1 - \O(n^{-1})$. Hence, the expected total squared loss of the light rows is at least: 
\begin{align*}
(n-s-k) (1 - \O(n^{-1})) \geq (n - (1+\alpha) \cdot k) - \O(n^{-1})
\end{align*} 
Hence, the expected squared loss of a CountSketch whose sparsity is picked at random is at least
\begin{align*}
{\ell k \over 4e} - \O(n^{-1}) + n - (1+\alpha)k - \O(n^{-1}) \geq n + {\ell k \over 4e} - (1+\alpha) k - \O(n^{-1}) 
\end{align*}
\end{proof}

\begin{corollary}\label{cor:comparison}
Let $s = \alpha (k-1)$ where $0.7<\alpha < 1$ and let $\ell \geq {(4e+1)n\over \alpha k}$. 
Let $S_g$ be the CountSketch whose sparsity pattern is learned over a training set drawn from $\sA_{sp}$ via the greedy approach. Let $S_r$ be a CountSketch whose sparsity pattern is picked uniformly at random. Then, for an $n\times d$ matrix $A\sim \sA_{sp}$ where $d = \Omega(n^6 \ell^2)$, the expected loss of the best $\rankk$ approximation of $A$ returned by $S_r$ is worse than the approximation loss of the best $\rankk$ approximation of $A$ returned by $S_g$ by at least a constant factor. 
\end{corollary}
\begin{proof}
\begin{align*}
\E_{S_r}[\min_{\rankk X\in \rowsp(S_r A)} \left\|X - A\right\|_F^2]  
& \geq n + {\ell k \over 4e} - (1+\alpha) k - n^{-\O(1)} &&\rhd \text{Lemma~\ref{lem:random-pattern}}\\ 
& \geq (1+1/\alpha) (n-s) &&\rhd \ell \geq {(4e+1)n\over \alpha k} \\
& = (1+1/\alpha) \min_{\rankk X\in \rowsp(S_g A)} \left\|X - A\right\|_F^2 &&\rhd \text{Corollary~\ref{cor:greedy-cost}}\\
\end{align*}
\end{proof}

\subsection{Zipfian on squared row norms.}\label{sec:greedy_zipfian}
Each matrix $A \in \mathbb{R}^{n \times d} \sim \sA_{zipf}$ has rows which are uniformly random and orthogonal. Each $A$ has $2^{i+1}$ rows of squared norm $n^2/2^{2i}$ for $i \in [1, \ldots, \O(\log(n))]$.
We also assume that each row has the same squared norm for all members of $\sA_{zipf}$. 

In this section, the $s$ rows with largest norm are called the {\em heavy} rows and the remaining are the {\em light} rows. 
For convenience, we number the heavy rows $1, \ldots, s$; however, the heavy rows can appear at any indices, as long as any row of a given index has the same norm for all members of $\sA_{zipf}$. 
Also, we assume that $s \leq k/2$ and, for simplicity, $s = \sum_{i=1}^{h_s} 2^{i+1}$ for some $h_s \in \mathbb{Z}^{+}$. That means the minimum squared norm of a heavy row is $n^2/2^{2h_s}$ and the maximum squared norm of a light row is $n^2/{2^{2h_s+2}}$. 

The analysis of the greedy algorithm ordered by non-increasing row norms on this family of matrices is similar to our analysis for the spiked covariance model. Here we analyze the case in which rows are orthogonal. By continuity, if the rows are close enough to being orthogonal, all decisions made by the greedy algorithm will be the same. 

As a first step, by Lemma~\ref{lem:heavy-rows}, at the end of iteration $k$ the first $k$ rows are assigned to different bins. Then, via a similar inductive proof, we show that none of the light rows are mapped to a bin that contains one of the top $s$ heavy rows.

\begin{lemma}\label{lem:zipf-no-collision-light}
At each iteration $k+r$, the greedy algorithm picks the position of the non-zero value in the $(k+r)$-th column of the CountSketch matrix $S$ so that the light row $A_{k+r}$ is mapped to a bin that does not contain any of top $s$ heavy rows.
\end{lemma} 
\begin{proof}
We prove the statement by induction. The base case $r=0$ trivially holds as the first $k$ rows are assigned to distinct bins. Next we assume that in none of the first $k+r-1$ iterations a light row is assigned to a bin that contains a heavy row. Now, we consider the following cases:

\paragraph{1. If $A_{k+r}$ is assigned to a bin that only contains light rows.}
Without loss of generality we can assume that $A_{k+r}$ is assigned to $b_k$. Since the vectors are orthogonal, we only need to bound the difference in the projection of $A_{k+r}$ and the light rows that are assigned to $b_k$ onto the direction of $w_k$ before and after adding $A_{k+r}$ to $b_k$. 
In this case, the total squared loss corresponding to rows in $b_k$ and $A_{k+r}$ before and after adding $A_{k+1}$ are respectively 
\begin{align*}
\textit{before adding $A_{k+r}$ to $b_k$: } 
&\left\| A_{k+r} \right\|^2_2 + \sum_{A_j \in b_k} \left\| A_{j} \right\|^2_2 - ({ \sum_{A_j \in b_k} \left\| A_{j} \right\|^4_2\over \sum_{A_j \in b_k} \left\| A_{j} \right\|^2_2 }) \\
\textit{after adding $A_{k+r}$ to $b_k$: } 
&\left\| A_{k+r} \right\|^2_2 + \sum_{A_j \in b_k} \left\| A_{j} \right\|^2_2 - ({\left\| A_{k+r} \right\|^4_2 +  \sum_{A_j \in b_k} \left\| A_{j} \right\|^4_2\over \left\| A_{k+r} \right\|^2_2 + \sum_{A_j \in b_k} \left\| A_{j} \right\|^2_2 })
\end{align*}
Thus, the amount of increase in the squared loss is 
\begin{align}
({ \sum_{A_j \in b_k} \left\| A_{j} \right\|^4_2\over \sum_{A_j \in b_k} \left\| A_{j} \right\|^2_2 }) - ({\left\| A_{k+r} \right\|^4_2 + \sum_{A_j \in b_k} \left\| A_{j} \right\|^4_2\over \left\| A_{k+r} \right\|^2_2 + \sum_{A_j \in b_k} \left\| A_{j} \right\|^2_2 }) 
&= {\left\| A_{k+r} \right\|^2_2 \cdot \sum_{A_j \in b_k} \left\| A_{j} \right\|^4_2 - \left\| A_{k+r} \right\|^4_2 \cdot \sum_{A_j \in b_k} \left\| A_{j} \right\|^2_2 \over (\sum_{A_j \in b_k} \left\| A_{j} \right\|^2_2) (\left\| A_{k+r} \right\|^2_2 + \sum_{A_j \in b_k} \left\| A_{j} \right\|^2_2)} \nonumber \\
&= \left\| A_{k+r} \right\|^2_2 \cdot { {\sum_{A_j \in b_k} \left\| A_{j} \right\|^4_2 \over \sum_{A_j \in b_k} \left\| A_{j} \right\|^2_2} - \left\| A_{k+r} \right\|^2_2  \over \sum_{A_j \in b_k} \left\| A_{j} \right\|^2_2 + \left\| A_{k+r} \right\|^2_2} \nonumber \\
&\leq \left\| A_{k+r} \right\|^2_2 \cdot { \sum_{A_j \in b_k} \left\| A_{j} \right\|^2_2 - \left\| A_{k+r} \right\|^2_2  \over \sum_{A_j \in b_k} \left\| A_{j} \right\|^2_2 + \left\| A_{k+r} \right\|^2_2} \label{eq:light} 
\end{align}

\paragraph{2. If $A_{k+r}$ is assigned to a bin that contains a heavy row.}  
Without loss of generality and by the induction hypothesis, we assume that $A_{k+r}$ is assigned to a bin $b$ that only contains a heavy row $A_j$. 
Since the rows are orthogonal, we only need to bound the difference in the projection of $A_{k+r}$ and $A_j$ 
In this case, the total squared loss corresponding to $A_j$ and $A_{k+r}$ before and after adding $A_{k+1}$ to $b$ are respectively 
\begin{align*}
\textit{before adding $A_{k+r}$ to $b_k$: } 
&\left\| A_{k+r} \right\|^2_2  \\
\textit{after adding $A_{k+r}$ to $b_k$: } 
&\left\| A_{k+r} \right\|^2_2 + \left\| A_{j} \right\|^2_2 - ({\left\| A_{k+r} \right\|^4_2 +  \left\| A_{j} \right\|^4_2\over \left\| A_{k+r} \right\|^2_2 + \left\| A_{j} \right\|^2_2 })
\end{align*}
Thus, the amount of increase in the squared loss is 
\begin{align}
\left\| A_{j} \right\|^2_2 - ({\left\| A_{k+r} \right\|^4_2 +  \left\| A_{j} \right\|^4_2\over \left\| A_{k+r} \right\|^2_2 + \left\| A_{j} \right\|^2_2 }) = \left\| A_{k+r} \right\|^2_2 \cdot {\left\| A_{j} \right\|^2_2 - \left\| A_{k+r} \right\|^2_2 \over  \left\| A_{j} \right\|^2_2 + \left\| A_{k+r} \right\|^2_2} \label{eq:heavy}
\end{align}
Then~\eqref{eq:heavy} is larger than~\eqref{eq:light} if $\left\| A_{j} \right\|^2_2 \geq \sum_{A_i \in b_k} \left\| A_i \right\|^2_2$. Next, we show that at every inductive iteration, there exists a bin $b$ which only contains light rows and whose squared norm is smaller than the squared norm of any heavy row. For each value $m$, define $h_m$ so that 
$m = \sum_{i=1}^{h_m} 2^{i+1} = 2^{h_m+2} - 2 $. 

Recall that all heavy rows have squared norm at least ${n^2 \over 2^{2h_s}}$. There must be a bin $b$ that only contains light rows and has squared norm at most 
\begin{align*}
\left\| w \right\|^2_2 = \sum_{A_i \in b} \left\| A_{i} \right\|^2_2
&\leq {n^2 \over 2^{2(h_s+1)}} + {\sum_{i=h_k+1}^{h_n} {2^{i+1} n^2 \over 2^{2i}}\over k-s} \\
&\leq {n^2 \over 2^{2(h_s+1)}} + {2 n^2 \over 2^{h_k} (k-s)} \\
&\leq {n^2 \over 2^{2(h_s+1)}} + {n^2 \over 2^{2h_k}} &&\rhd s \leq k/2 \text{ and } k> 2^{h_k +1}\\
&\leq {n^2 \over 2^{2h_s+1}} &&\rhd h_k \geq h_s+1\\
&< \left\| A_s \right\|^2_2
\end{align*} 
Hence, the greedy algorithm will map $A_{k+r}$ to a bin that only contains light rows.
\end{proof}

\begin{corollary}\label{cor:zipf-greedy-cost}
The squared loss of the best $\rankk$ approximate solution in the rowspace of  $S_gA$ for $A \in \mathbb{R}^{n \times d} \sim \sA_{zipf}$ where $A\in \mathbb{R}^{n\times d}$ and $S_g$ is the CountSketch constructed by the greedy algorithm with non-increasing order, is $<{n^2 \over 2^{h_k -2}}$.
\end{corollary}
\begin{proof}
At the end of iteration $k$, the total squared loss is $\sum_{i=h_k +1}^{h_n} 2^{i+1} \cdot {n^2 \over 2^{2i}}$. After that, in each iteration $k+r$, by~\eqref{eq:light}, the squared loss increases by at most $\left\| A_{k+r} \right\|^2_2$. Hence, the total squared loss in the solution returned by $S_g$ is at most  
\begin{align*}
2 (\sum_{i=h_k +1}^{h_n} {2^{i+1} n^2 \over 2^{2i}}) 
= {4n^2} \cdot \sum_{i=h_k +1}^{h_n} {1 \over 2^i} < {4n^2 \over 2^{h_k}} = {n^2 \over 2^{h_k -2}}
\end{align*}  
\end{proof}

Next, we bound the squared loss of the best $\rankk$-approximate solution constructed by the standard CountSketch with a randomly chosen sparsity pattern.
\begin{observation}\label{lem:lower-bound-loss}
Let us assume that the orthogonal rows $A_{r_1}, \cdots, A_{r_c}$ are mapped to the same bin and for each $i\leq c$, $\left\|A_{r_1}\right\|_2^2 \geq \left\|A_{r_i}\right\|_2^2$. Then, the total squared loss of $A_{r_1}, \cdots, A_{r_c}$ after projecting onto $A_{r_1} \pm \cdots \pm A_{r_c}$ is at least $\left\| A_{r_2} \right\|^2_2 + \cdots + \left\| A_{r_c} \right\|^2_2$. 
\end{observation}
\begin{proof}
Note that since $A_{r_1}, \cdots, A_{r_c}$ are orthogonal, for each $i\leq c$, the squared projection of $A_{r_i}$ onto $A_{r_1} \pm \cdots \pm A_{r_c}$ is ${\left\| A_{r_i} \right\|^4_2 / \sum_{j=1}^c \left\| A_{r_j} \right\|^2_2}$. Hence, the sum of squared projection coefficients of $A_{r_1}, \cdots, A_{r_c}$ onto $A_{r_1} \pm \cdots \pm A_{r_c}$ is 
\begin{align*}
{\sum_{j=1}^c \left\| A_{r_j} \right\|^4_2 \over \sum_{j=1}^c \left\| A_{r_j} \right\|^2_2} \leq \left\| A_{r_1} \right\|^2_2 
\end{align*}
Hence, the total projection loss of $A_{r_1}, \cdots, A_{r_c}$ onto $A_{r_1} \pm \cdots \pm A_{r_c}$ is at least 
\begin{align*}
\sum_{j=1}^c \left\| A_{r_j} \right\|^2_2 - \left\| A_{r_1} \right\|^2_2 = \left\| A_{r_2} \right\|^2_2 + \cdots + \left\| A_{r_c} \right\|^2_2.
\end{align*}
\end{proof}
In particular, Observation~\ref{lem:lower-bound-loss} implies that whenever two rows are mapped into the same bin, the squared norm of the row with smaller norm {\em fully} contributes to the total squared loss of the solution.	

\begin{lemma}\label{lem:zipf-random-pattern}
For $k> 2^{10} -2$, the expected squared loss of the best $\rankk$ approximate solution in the rowspace of $S_r A$ for $A_{n\times d}\sim \sA_{zipf}$, where $S_r$ is the sparsity pattern of a CountSketch chosen uniformly at random, is at least ${1.095 n^2 \over 2^{h_k -2}}$.
\end{lemma} 
\begin{proof}
In light of Observation~\ref{lem:lower-bound-loss}, we need to compute the expected number of collision between rows with ``large'' norm. We can interpret the randomized construction of the CountSketch as a ``balls and bins'' experiment. 

For each $0\leq j\leq h_k$, let $\sA_j$ denote the set of rows with squared norm ${n^2 \over 2^{2(h_k -j)}}$ and let $\sA_{>j} = \bigcup_{j < i \leq h_k} \sA_i$. Note that for each $j$, $|\sA_{j}| = 2^{h_k -j +1}$ and $|\sA_{>j}| = \sum_{i=j+1}^{h_k} 2^{h_k - i +1} = \sum_{i=1}^{h_k-j} 2^{i} =2 (2^{h_k -j} - 1)$. Moreover, note that $k = 2 (2^{h_k +1} -1)$.
Next, for a row $A_r$ in $\sA_j$ ($0\leq j<h_k$), we compute the probability that at least one row in $\sA_{>j}$ collides with $A_r$.
\begin{align*}
\text{\bf Pr}[\text{at least one row in $\sA_{>j}$ collide with $A_r$}] 
&= (1 - (1-{1\over k})^{|\sA_{>j}|}) \\
&\geq (1 - e^{-{|\sA_{>j}|\over k}}) \\
&= (1 - e^{-{2^{h_k -j} -1 \over 2^{h_k+1} -1}}) \\
&\geq (1 - e^{-2^{-j-2}}) &&\rhd\text{since } {2^{h_k -j} -1 \over 2^{h_k+1} -1} > 2^{-j-2} 
\end{align*}   
Hence, by Observation~\ref{lem:lower-bound-loss}, the contribution of rows in $\sA_j$ to the total squared loss is at least
\begin{align*}
(1 - e^{-2^{-j-2}}) \cdot |\sA_j| \cdot {n^2 \over 2^{2(h_k -j)}} 
=& (1 - e^{-2^{-j-2}}) \cdot {n^2 \over 2^{h_k -j -1}}
= (1 - e^{-2^{-j-2}}) \cdot {n^2 \over 2^{h_k - 2}} \cdot 2^{j-1}
\end{align*}
Thus, the contribution of rows with ``large'' squared norm, i.e., $\sA_{>0}$, to the total squared loss is at least\footnote{The numerical calculation is computed using WolframAlpha.}
\begin{align*}
{n^2 \over 2^{h_k -2}} \cdot \sum_{j=0}^{h_k} 2^{j-1} \cdot (1- e^{-2^{-j-2}}) \geq 1.095 \cdot {n^2 \over 2^{h_k -2}} &&\rhd\text{for $h_k> 8$}
\end{align*}
\end{proof}

\begin{corollary}\label{cor:zipf-comparison}
Let $S_g$ be a CountSketch whose sparsity pattern is learned over a training set drawn from $\sA_{sp}$ via the greedy approach. Let $S_r$ be a CountSketch whose sparsity pattern is picked uniformly at random. Then, for an $n\times d$ matrix $A\sim \sA_{zipf}$, for a sufficiently large value of $k$, the expected loss of the best $\rankk$ approximation of $A$ returned by $S_r$ is worse than the approximation loss of the best $\rankk$ approximation of $A$ returned by $S_g$ by at least a constant factor. 
\end{corollary}
\begin{proof}
The proof follows from Lemma~\ref{lem:zipf-random-pattern} and Corollary~\ref{cor:zipf-greedy-cost}.
\end{proof}

\begin{remark}
We have provided evidence that the greedy algorithm that examines the rows of $A$ according to a non-increasing order of their norms (i.e., {\em greedy with non-increasing order}) results in a better $\rankk$ solution compared to the CountSketch whose sparsity pattern is chosen at random. However, still other implementations of the greedy algorithm may result in a better solution compared to the greedy algorithm with non-increasing order. To give an example, in the following simple instance the greedy algorithm that checks the rows of $A$ in a random order (i.e., {\em greedy with random order}) achieves a $\rankk$ solution whose cost is a constant factor better than the solution returned by the greedy with non-increasing order.

Let $A$ be a matrix with four orthogonal rows $u,u, v, w$ where $\left\| u \right\|_2 =1$ and $\left\| v \right\|_2 = \left\| w \right\|_2 =1+\eps$ and suppose that the goal is to compute a $\mathrm{rank}$-$2$ approximation of $A$. Note that in the greedy algorithm with non-decreasing order, $v$ and $w$ will be assigned to different bins and by a simple calculation we can show that the copies of $u$ also will be assigned to different bins. 
Hence, the squared loss in the computed $\mathrm{rank}$-$2$ solution is $1 + {(1+\eps)^2 \over 2 + (1+\eps)^2}$.  
However, the optimal solution will assign $v$ and $w$ to one bin and the two copies of $u$ to the other bin which results in a squared loss of $(1+\eps)^2$ which is a constant factor smaller than the solution returned by the greedy algorithm with non-increasing order for sufficiently small values of $\eps$. 

On the other hand, in the greedy algorithm with random order, with a constant probability of (${1\over 3} + {1\over 8}$), the computed solution is the same as the optimal solution. Otherwise, the greedy algorithm with random order returns the same solution as the greedy algorithm with a  non-increasing order. Hence, in expectation, the solution returned by the greedy with random order is better than the solution returned by the greedy algorithm with non-increasing order by a constant factor. 
\end{remark}
\end{document}